\documentclass{article}

\usepackage[final,nonatbib]{neurips_2022}

\usepackage{comment}
\usepackage{microtype}
\usepackage{graphicx}
\usepackage{subfigure}
\usepackage{booktabs} % for professional tables

\usepackage{hyperref}

\usepackage{amsmath}
\usepackage{amssymb}
\usepackage{mathtools}
\usepackage{amsthm}
\usepackage{algorithm}
\usepackage{algorithmic}
\usepackage{hyperref}
\usepackage{url}
\usepackage{stmaryrd}
\usepackage{graphicx}
\usepackage{makecell}

\usepackage{amssymb}
\usepackage{xspace}
\usepackage{lipsum}
\usepackage{mathtools}
\usepackage{cuted}
\usepackage{caption}
\usepackage[capitalize,noabbrev]{cleveref}

\usepackage{amsmath}
\usepackage{amssymb}
\usepackage{mathtools}
\usepackage{amsthm}
\usepackage{microtype}
\usepackage{graphicx}
\usepackage{subfigure}
\usepackage{booktabs} % for professional tables

\usepackage{hyperref}

\usepackage[utf8]{inputenc} % allow utf-8 input
\usepackage[T1]{fontenc}    % use 8-bit T1 fonts
\usepackage{hyperref}       % hyperlinks
\usepackage{url}            % simple URL typesetting
\usepackage{booktabs}       % professional-quality tables
\usepackage{amsfonts}       % blackboard math symbols
\usepackage{nicefrac}       % compact symbols for 1/2, etc.
\usepackage{microtype}      % microtypography
\usepackage{xcolor}         % colors
\usepackage{wrapfig}

\usepackage{amssymb}
\usepackage{xspace}
\usepackage{lipsum}
\usepackage{mathtools}
\usepackage{cuted}
\usepackage{caption}
\usepackage[capitalize,noabbrev]{cleveref}
\usepackage{placeins}
\usepackage[resetlabels,labeled]{multibib}
\newcites{A}{References}

\theoremstyle{plain}
\newtheorem{theorem}{Theorem}[section]

\theoremstyle{definition}
\newtheorem{definition}[theorem]{Definition}

\theoremstyle{remark}

\let\emptyset\varnothing

\newcommand{\Tabref}[1]{Table~\ref{#1}}

\newcommand{\Figref}[1]{Figure~\ref{#1}}

\def\x{x}
\def\tx{\tilde{x}}
\def\tm{\tilde{m}}
\def\tc{\tilde{c}}
\def\tq{\tilde{q}}
\def\G{\mathcal{G}}
\def\R{\mathcal{R}}

\def\C{\mathcal{C}}
\def\M{\mathcal{M}}

\definecolor{ourRed}{HTML}{E24A33}
\definecolor{ourBlue}{HTML}{348ABD}
\definecolor{ourPurple}{HTML}{988ED5}
\definecolor{ourGray}{HTML}{777777}
\definecolor{ourLightGray}{HTML}{B8B8B8}
\definecolor{ourYellow}{HTML}{FBC15E}
\definecolor{ourGreen}{HTML}{4D8951}
\definecolor{ourPink}{HTML}{FFB5B8}
\definecolor{oursteelblue}{HTML}{9BB8D7}
\definecolor{ourOrange}{HTML}{FDBA58}
\definecolor{ourWhite}{HTML}{FAFAFA}

\newcommand{\xhdr}[1]{\vspace{2mm}\noindent{{\bf #1.}}}
\newcommand{\xhdrx}[1]{\vspace{0mm}\noindent{{\bf #1.}}}

\newcommand{\proj}{ZeroC\xspace}
\newcommand{\projs}{HC-EBMs\xspace}
\newcommand{\cebm}{HC-EBM\xspace}
\newcommand{\cebms}{HC-EBMs\xspace}
\newcommand{\rebm}{R-EBM\xspace}
\newcommand{\rebms}{R-EBMs\xspace}

\newcommand{\projfull}{Zero-shot Concept Recognition and Acquisition\xspace}

\title{ZeroC: A Neuro-Symbolic Model for \\
Zero-shot Concept Recognition
and Acquisition \\
at Inference Time}

\author{%
 Tailin Wu\\
  Department of Computer Science\\
  Stanford University\\
  \texttt{tailin@cs.stanford.edu} \\
  \And
  Megan Tjandrasuwita\\
  Department of Computer Science\\
  California Institute of Technology\\
  \texttt{megantj@mit.edu} \\
  \And
   Zhengxuan Wu\\
  Department of Computer Science\\
  Stanford University\\
  \texttt{wuzhengx@cs.stanford.edu} \\
  \And
   Xuelin Yang\\
  Department of Computer Science\\
  Stanford University\\
  \texttt{xyang23@cs.stanford.edu} \\
  \And
   Kevin Liu\\
  Department of Computer Science\\
  Stanford University\\
  \texttt{liuk@cs.stanford.edu} \\
  \And
   Rok Sosi\v{c}\\
  Department of Computer Science\\
  Stanford University\\
  \texttt{rok@cs.stanford.edu} \\
   \And
    Jure Leskovec\\
    Department of Computer Science\\
  Stanford University\\
  \texttt{jure@cs.stanford.edu} \\
}

\begin{document}

\maketitle

\begin{abstract}
Humans have the remarkable ability to recognize and acquire novel visual concepts in a zero-shot manner. Given a high-level, symbolic description of a novel concept in terms of previously learned visual concepts and their relations, humans can recognize novel concepts without seeing any examples. Moreover, they can acquire new concepts by parsing and communicating symbolic structures using learned visual concepts and relations. Endowing these capabilities in machines is pivotal in improving their generalization capability at inference time.
In this work, we introduce \projfull (\proj), a neuro-symbolic architecture that can recognize and acquire novel concepts in a zero-shot way.  \proj represents concepts as graphs of constituent concept models (as nodes) and their relations (as edges). To allow inference time composition, we employ energy-based models (EBMs) to model concepts and relations. We design \proj architecture so that it allows a one-to-one mapping between a symbolic graph structure of a concept and its corresponding EBM, which for the first time, allows acquiring new concepts, communicating its graph structure, and applying it to classification and detection tasks (even across domains) at inference time. We introduce algorithms for learning and inference with \proj. We evaluate \proj on a challenging grid-world dataset which is designed to probe zero-shot concept recognition and acquisition, and demonstrate its capability.
\footnote{
Project website and code can be found at \url{http://snap.stanford.edu/zeroc/}.}.

\end{abstract}

\section{Introduction}

Humans learn in diverse ways. Besides learning from demonstrations of a novel concept, humans can also learn concepts on a high-level. Consider learning the ``rectangle'' concept, for example. Suppose that one has never seen such a concept, but has already mastered the concept of ``line'' and relations of ``parallel'' and ``perpendicular''; s/he can easily master the ``rectangle'' concept when told that a ``rectangle'' consists of two pairs of ``lines'', the lines \emph{within} the pairs are ``parallel,'' and the lines \emph{between} the pairs are ``perpendicular''. Then s/he can directly use this newly mastered concept to recognize ``rectangles'' in novel images. 

Such zero-shot concept \emph{recognition} capability is still beyond the reach of deep learning models, which require many examples (as in typical supervised learning), or many tasks drawn from the same distribution (as in few-shot learning) to learn a novel concept. Moreover, along with the above symbolic to instance direction, humans can easily do the \emph{reverse}. Suppose again that we haven't seen the ``rectangle'' concept, but have already mastered concepts of ``line'' and relations of ``parallel'' and ``perpendicular''. When seeing a novel image which contains an instance of a rectangle, humans can easily decompose it into its constituent concepts of ``lines'' and their relational graph structure (from instance to symbolic). Moreover, humans can then communicate with each other about this new concept, allowing the transfer of knowledge in a high-level and in a succinct way. Such zero-shot concept \emph{acquisition} capability at inference time is beyond the reach of today's machine learning and AI systems. Above all, endowing the above two capabilities of zero-shot concept recognition and acquisition capability to machines at inference time, will allow them to tackle more diverse tasks.

\begin{figure}[t]
\centering
    \includegraphics[width=\textwidth]{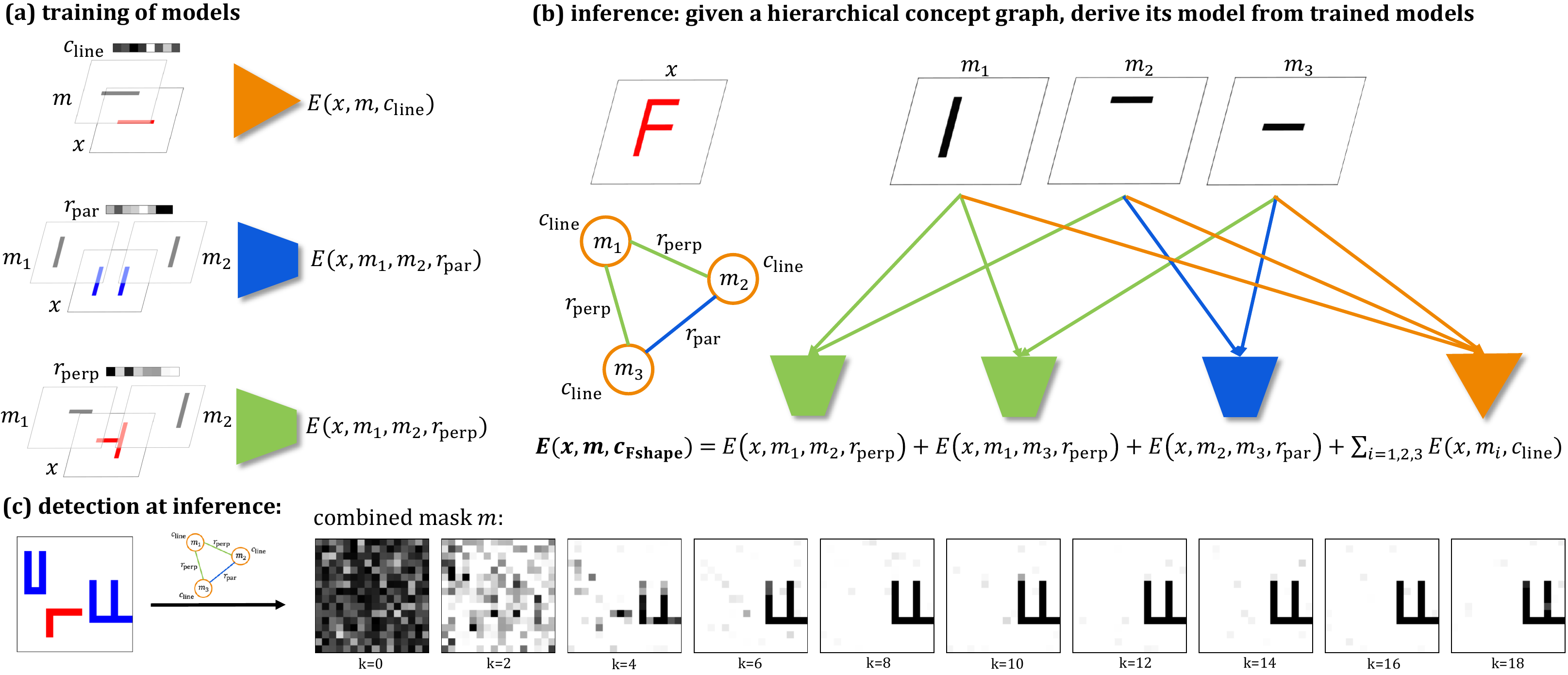}
    % \vspace{-1mm}
    \caption{{\bf Composition of Concept Models.} We demonstrate the model composition for a novel, hierarchical concepts on the recognition of letter F. {\bf (a)} During training, we learn the models for constituent concepts, the concept "line" in this case, and relations, which are "parallel" and "perpendicular". {\bf (b)} During inference, we take the concept graph of F and use it to derive the model for F from the models of its constituents. Note that no training is performed on the hierarchical concept F. {\bf (c)} Example of detecting hierarchical concept on pixel level with \proj during inference in Sec. \ref{sec:experiments}, using Eq. (\ref{eq:SGLD_m}).
    }
    \label{fig:method}
    
\end{figure}

Past efforts have attempted to address aspects of the above problem of recognizing and acquiring novel concepts at inference time in a zero-shot manner. An essential component is compositionality, \textit{i.e.} the ability to compose new concepts from elementary ones. \cite{du2020compositional} and \cite{higgins2017scan} introduce techniques for compositional generation in the context of energy-based models (EBMs) and variational autoencoders (VAEs), respectively. However, these works only consider composing factors of variation (\textit{e.g.}, color, position, smiling, young). In order to recognize concepts that consist of constituent concepts as parts, the structure of objects and their relationships are equally important. 
For instance, given 4 ``line'' concepts, without specifying their relational graph structure, they can easily form a trapezoid or a generic parallelogram. \cite{shanahan2020explicitly} designed a transformer-based relational architecture for learning explicit relations, and showed that it can generalize to new objects for the relation. However, this architecture cannot generalize to concepts with more complex relational graph structure. 

It remains an open problem of how to design an architecture that has the capability to compose novel concepts at inference time, based on its internal relational graph structure of prior learned concepts and relations. A separate line of works \cite{romera2015embarrassingly,jiang2017learning,bucher2017generating,zhao2018zero,schonfeld2019generalized} tackle zero-shot classification, but are limited to novel composition of label features for concepts, and ignore relation features as well as the graph structure formed by concepts and relations.

Here we introduce our \proj architecture, which is a model designed towards zero-shot concept recognition and acquisition. It models a visual concept or relation as an Energy-Based Model (EBM) with 3 inputs: an image, a mask (or two masks for relation) denoting a set of pixels within the image, and a concept/relation label string. It returns a low scalar energy value if the mask(s) correctly indicates the concept or relation within the image, and a high energy value otherwise.
After training, the concept and relation EBMs are able to perform classification of concept/relation labels given an image, and detection of concept/relation mask(s) given an image and concept/relation label. The \emph{core} contribution of this work is \proj's architecture and inference algorithms that allow a one-to-one correspondence between a specification of a relational graph structure of a novel concept and its corresponding composed concept model that can also perform similar downstream classification and detection, which enables zero-shot concept recognition and acquisition at inference time. 

To perform zero-shot concept recognition at inference time, \proj can simply compose a new concept EBM, by summing over the constituent concept and relation EBMs according to the specified graph structure. This new concept EBM has the same input-output format as its constituent concept EBMs, allowing the same capability of classification and detection. Note that through the above process, \proj have mastered a concept without seeing an image, only symbolic instruction, thus it is called ``zero-shot''. To perform zero-shot concept acquisition, we introduce an algorithm that   given an instance (image) of a novel concept, parse it into a relational graph of prior learned concepts and relations. Such relational graph can then be transferred, even cross-domain, allowing models independently trained in a different domain to be able to zero-shot classify and detect this concept in its own domain.

There are currently no specialized benchmarks to test such zero-shot concept recognition and acquisition capabilities. Inspired by the Abstraction and Reasoning Corpus (ARC)~\cite{chollet2019measure}, we have created a synthetic, grid-world based dataset with tasks that capture the essence of the above capabilities that are deceptively simple for humans, but very hard for neural models.
We demonstrate that our model, trained to classify and detect elementary visual concepts and relations, is able to classify and detect novel concepts at inference time, being given a zero-shot symbolic graph structure. We also compare with a state-of-the-art zero-shot learning model CADA-VAE \cite{schonfeld2019generalized}. Due to that it is not designed for such graph-based composition, it significantly underperforms our \proj. 
Furthermore, we show that two independently trained models, one on 2D images and the other on 3D images, are able to acquire hierarchical concepts from each other by communicating the graph structure and then perform classification and detection of these hierarchical concepts in their respective domains.

\vspace{-1mm}
\section{Method}
\vspace{-1mm}

In this section, we describe the \proj architecture for recognizing and acquiring novel visual concepts at inference time. We give an overview of our approach, describe how it performs classification and detection, how it supports zero-shot concept recognition, how independently-trained models can acquire hierarchical concepts from each other and, finally, introduce its learning method.

\subsection{An Overview of \proj Architecture}
\label{sec:overview-architecture}

The key components of \proj are concepts and relations.
Each concept consists of a graph and an energy-based model. The concept graph describes the concept  as a composition of its constituent concepts and relations between them.
The base concepts that do not have any constituent concepts are called elementary concepts. Their graph is a singleton.
The concept energy-based model is used to recognize the concept in the input data.
Each relation is represented in \proj also with a graph and an energy-based model. The relation graph is simply an edge that connects the two related concepts. A hierarchical concept is a concept composed of constituent concepts as nodes and relations as edges according to a graph structure.

During training, energy-based models are learned from images and concept or relation labels.
During inference, the learned models are used to recognize concepts seen during training or more complex hierarchical concepts which were not seen before.
For concepts seen during training, their learned models are applied.
For new hierarchical concepts not seen during training (\textit{i.e.}, without learned model), their models are derived from the graph of the new hierarchical concept and energy-based models of their constituent concepts and relations.
Concepts and relations in \proj can be viewed as templates for objects and their connections, which then get grounded during inference with a specific image, where those objects and relations are assigned actual values.
\proj also handles object masks, which indicate object locations in the image.

\subsection{\projfull}
\label{sec:architecture}

Formally, we model a concept with $E_{X,M,C}(x,m,c)$\footnote{We will use capital letters (\textit{e.g.} $X, M$) to represent random variables, and small ones (\textit{e.g.} $x, m$) to represent their instances.}, which maps an image $x\in \mathbb{R}^D$, a mask $m\in[0,1]^D$ and a concept label\footnote{The concept label $c$ is a categorical variable, that we use to refer to a concept, like ``line'', ``rectangle'', or ``cat''. In the following, if without confusion, we will refer to a concept using its label $c$.} $c$ to a scalar energy.
Similarly, we model relations with $E_{X,M_1,M_2,R}(x,m_1,m_2,r)$, where $m_1$ and $m_2$ are a pair of masks indicating two objects in the image $x$, and $r$ is a relation label between the objects. 
The concept and relation models have a probabilistic interpretation. For example, the energy function $E_{X,M,C}(x,m,c)$ corresponds to the joint probability of 
$$P_{X,M,C}(x,m,c)=\frac{1}{Z_C}\text{exp}\left(-E_{X,M,C}(x,m,c)\right)$$
where $Z_C=\sum_{c\in C}\int e^{-E_{X,M,C}(x,m,c)}\;dx\;dm$ is a normalizing constant. Therefore, if the mask $m$ is actually masking an object that belongs to concept $c$ in image $x$, then $P_{X,M,C}(x,m,c)$ will be high and $E_{X,M,C}(x,m,c)$ will be low, and vice versa.
Essentially, the models define an energy landscape for their respective multi-modal inputs that give low energy if the mask correctly identifies its corresponding concept. 

Next, we show how \proj performs detection, classification, and models hierarchically composed concepts. We use the term \cebm to denote concept models and \rebm for relation models.

\xhdrx{Detection} We want to infer the location mask $m$ of the concept $c$, given image $x$. Probabilistically, we are computing $P_{M|X,C}(m|x,c)$. To perform detection, we employ Stochastic Gradient Langevin Dynamics (SGLD)  to sample  masks on the energy landscape \cite{du2019implicit}:
\begin{equation}
\label{eq:SGLD_m}
\tilde{m}^{k}=\tilde{m}^{k-1}-\frac{\lambda}{2}\nabla_{m}E_{X,M,C}(x,\tilde{m}^{k-1},c)+\omega^{k}, \ \omega^{k}\sim \mathcal{N}(0,\lambda)
\end{equation}

where $\tilde{m}^k$ is the inferred mask at the $k$th iteration, $k=1,2,...K$. Applying \cite{welling2011bayesian}, as $K\to+\infty$ and $\lambda\to 0$, we generate samples from the distribution of $P_{M|X,C}(m|x,c)=\frac{1}{Z_{x,c}}\text{exp}(-E_{X,M,C}(x,m,c))$, where $Z_{x,c}=\int \, \text{exp}(-E_{X,M,C}(x,m,c))\;dm$ is a normalizing constant. In practice, we use a finite $K$ to generate samples\footnote{In this paper, we use $i$ to index different concepts, $j$ to index relations, and $n$ to index example images.} $\tilde{m}^K_n, n=1,2,...N$ given $x, c$.

\xhdrx{Classification} We want to determine whether the concept $c$ appears in a given image $x$, \textit{i.e.} compute $P_{C|X}(c|x)$. We need to marginalize over the mask $m$:
\vspace{-1mm}
\begin{align*}
P_{C|X}(c|x)=&\frac{P_{X,C}(x,c)}{P_X(x)}=\frac{\int \, P_{X,M,C}(x,m,c) \;dm}{\sum_{c\in C}\int \,P_{X,M,C}(x,m,c)\;dm}\\
=&\frac{\int \, \text{exp}\left(-E_{X,M,C}(x,m,c)\right)\;dm}{\sum_{c\in C}\int \, \text{exp}\left(-E_{X,M,C}(x,m,c)\right) \;dm}
\vspace{-2mm}
\end{align*}
We again use SGLD \cite{du2019implicit} in Eq. (\ref{eq:SGLD_m}) to generate $N$ samples $\tilde{m}^K_n, n\in[N]=\{1,2,...N\}$ given $x,c$, and approximate the above integral using maximum a posteriori (MAP) estimation:

\vspace{-4mm}
\begin{align}
\label{eq:classify}
P_{C|X}(c|x)\simeq\frac{\max_{n\in[N]} \text{exp}\left(-E_{X,M,C}(x,\tilde{m}^K_n,c)\right)}{\sum_{c\in C}\max_{n\in[N]} \text{exp}\left(-E_{X,M,C}(x,\tilde{m}^K_n,c)\right)}
\end{align}

In practice, we only need to  find the concept with the highest value in the numerator.

\xhdrx{Zero-shot recognition of novel concepts}
To master a novel hierarchical concept and directly use it for  classification and detection only given its relational graph structure, we need a way to compose the previous concept and relation energy based models. Here we introduce the hierarchical composition rule, using an English letter ``F'' as an illustrative example\footnote{For clarity, the concept graph shown has been simplified from the real model used in the experiments.}. 

The  concept F has one constituent concept, a "line", and two relations, "parallel" and "perpendicular". The models for constituents are learned during training. During inference, these models are combined, using the concept graph, into a combined energy model for the letter F. Essentially, the models of all the recognized constituent concepts and relations are added together. Note that  although F contains three lines, only one "line" model is needed. The concept graph acts as a template and recognized line instances (objects) are matched with nodes in the template to obtain actual models. The "line" model is instantiated three times with specific values for the three identified lines and the three models plus their corresponding relations models are used to derive the hierarchical model. In addition to the models, we also need to combine the masks of all the recognized constituent objects.

Formally, we define the following composition rule for \cebms plus the masks to be combined.
\begin{definition}
\label{def:composition_rule}\textbf{Hierarchical Composition Rule:}  Let a hierarchical concept $c$ have graph $\G=(\C,\R)$, where $\C$ are constituent concept nodes and $\R$ are relation edges. During inference, these nodes and edges are matched to recognized concept objects and their relationships, which provides their models and masks. The combined model $E_{X,M,C}(x,m,c)$ is then a sum of the models for all the nodes and edges in the graph and the mask $M$ is the maximum of all the masks:
\vspace{-0.5mm}
\begin{align}
\label{eq:composition}
&E_{X,M,C}(x,m,c)=\sum_{c_i\in\mathcal{C}} E_{X,M_i,C}(x,m_i,c_i)+\sum_{r_j\in\mathcal{R}} E_{X,M_{j1},M_{j2},R}(x,m_{j1},m_{j2},r_j)\\
&M:=\text{max}\{\{M_i\}, \{M_{j1}\}, \{M_{j2}\}\}\nonumber
\end{align}

\end{definition}

\subsection{Zero-shot Concept Acquisition at Inference Time}
\label{sec:transfer}

Here, we introduce techniques for acquiring novel concept graph, which is useful for sharing high-level knowledge transfer between independently trained models. Such capability may have implications in future scenarios where a hypothetical self-driving car communicating the structure of a novel road sign to other cars on the road, or a 2D vision model learning a novel object and transferring it to a more accurate 3D depth model for inference. To achieve this, we need to establish an equivalence relation between  concepts in different domains (\textit{e.g.} in 2D and 3D images).

\begin{definition}
\textbf{Structural Equivalence for \cebm:} Two \projs are structurally equivalent, if their graphs are isomorphic, and their constituent \cebms are recursively structurally equivalent.
\end{definition}
\vspace{-1mm}
For example, \cebms for a ``rectangle'' concept in 3D and 2D images are structurally equivalent if they have the same decomposition into two ``parallel-line'' \cebms connected by a ``connect'' \rebm, and the ``parallel-lines'' have the same decomposition into two ``line'' \cebms connected by a ``parallel'' \rebm. Even though the \cebms are grounded in different domains, they have the same abstract structure, so they represent the same hierarchical concept.

With structural equivalence, independently-trained \cebms from different domains can acquire hierarchical concepts via the graph structure. The algorithm works as follows (see Appendix~\ref{app:acquire}): the image is parsed to decompose it into a graph of concepts and relations previously learned by HC-EBM and R-EBM, then the graph is used in a different domain to compose its HC-EBM for the new concept. The key step of this process is the first, parsing step.

The parsing step can be seen as an inverse of the Hierarchical Composition Rule and is critical in allowing the model to recognize novel hierarchical concepts in its domain using prior-learned concepts and relations from other domains, thus facilitating downstream knowledge transfer between domains. Alg. \ref{alg:parsing} provides further details. Steps 1-2 infer the concept instances in image $x$ using SGLD (Eq. \ref{eq:SGLD_m}). Here the energy $E^{(p)}(x,\M_C,c)$ is the summation of HC-EBMs on independent masks, each mask $m_{i_l}$ representing a potential concept instance belonging to concept $c_i$. $E^{\text{(overlap)}}=\text{max}\left(0, \sum_{m_{i_n}\in \M_C} m_{i_n}-1\right)$ penalizes overlapping masks. Since the image probably contains fewer concept instances than given, some masks are empty with all-zero values; Step 3 removes these. Step 4 classifies relations (using Eq. \ref{eq:classify}) between pairs of detected concept instances. Step 5 combines the detected concept instances and their relations to build the graph $\G$ for this new hierarchical concept $c$.

\begin{algorithm}[t]
\caption{\textbf{Parsing Hierarchical Concept From Image}}
\label{alg:parsing}
\begin{algorithmic}
\STATE {\bfseries Require:} \cebm $E^{(C)}$ with prior-learned concepts $\{c_1,c_2,...c_I\}$, R-EBM model $E^{(R)}$ with prior-learned relations $\{r_1,r_2,...r_J\}$.
\STATE {\bfseries Require:} Image $x$, containing unseen hierarchical concept $c$.
\STATE {\bfseries Require:} Maximum number of instances $N_i$ for each concept $c_i$, $i=1,2,...I$.

\STATE 1: $E^{(p)}(x,\M_C,c):=\sum_{i=1}^I \sum_{n=1}^{N_i}E^{(C)}(x,m_{i_n},c_i)$
\STATE  $\ \ \ \ +\lambda_1 E^{\text{(overlap)}}(\M_C)$, where $\M_C:=\{m_{i_n}\},i_n=$ 
\STATE \ \ \ \  $1,...N_i\ \text{for} \ i=1,...I$, is the set of all masks.
\STATE 2: $\M_C\gets \textbf{SGLD}_{\M_C}(E^{(p)}(x,\M_C,c))$ \ \ // \textit{Using Eq. \ref{eq:SGLD_m}}
\STATE 3: $\M_C\gets$ \text{Remove-empty-masks}$(\M_C)$
\STATE 4: $\R\gets\{r_{j_1j_2}| \textbf{Classify}(E^{(R)};x,m_{j_1},m_{j_2}), \forall (m_{j_1},m_{j_2})\in\M_C\}$ \ \ // \textit{Using Eq. \ref{eq:classify} to classify}
\STATE 5: $\G\gets\text{Build-Graph}(\M_C,\R)$% \ \ //\textit{\ Build the structure graph using $\M_C$ and $\R$}
\STATE 6: \textbf{return} $\G$ 
\end{algorithmic}
\end{algorithm}

\subsection{Learning}
\label{sec:learning}

Since a standard EBM-training method yields poor performance as shown later in  our experiments, we describe here our approach to train the models. Further details of our training algorithm are given in Alg. \ref{alg:learning} in Appendix \ref{app:learning}. We are optimizing the following objective:
\begin{align}
\label{eq:objective}
L=\frac{1}{N}\sum_{n=1}^N\big(L^\text{(Improved)}_n +\alpha_\text{pos-std}L_n^\text{(pos-std)}+\alpha_\text{em}L^\text{(em)}_n+\alpha_\text{neg}L^\text{(neg)}_n\big)
\end{align}

The  expression of each term is given in Appendix \ref{app:regularizations}. The first loss $L^\text{(Improved)}$
is the objective proposed in \cite{du2020improved}, a state-of-the-art EBM training technique. However, to enable the challenging zero-shot concept recognition and acquisition, we need more suitable inductive biases. The next three terms inject the right inductive bias for the task. $L^\text{(pos-std)}$ makes sure that the positive energy have similar level, so the composed concept EBM can identify all constituent concepts and relations, without one constituent EBM having too low energy and only recognizing it. The empty-mask regularization $L^\text{(em)}$ makes sure that the when the mask is empty, its energy is between the positive and negative energy, which we prove in Appendix \ref{app:empty_mask_proof} that it is the necessary condition for correctly discover the underlying graph in Sec. \ref{sec:transfer}. The $L^\text{(neg)}$ additionally provide additional negative examples, encouraging discovering full concept instead of a part of it (see Appendix~\ref{app:regularizations} for details).

\section{Experiments}
\label{sec:experiments}

In this section, we set out to answer the following two questions: (a) given a specification of graph structure for a novel hierarchical concept, can the composed concept model successfully perform classification and detection tasks? (b) Can HC-EBM acquire a novel concept from HC-EBM trained from another domain?

We also evaluated our model in a challenging setting by comparing with other existing approaches in a controlled and systematic way.
Since there are no suitable dataset, we designed a Hierarchical-Concept corpus, a dataset based on grid-world images.  
The dataset comes with testing images for the classification and detection tasks which contain more complex, hierarchical concepts, composed from simpler concepts in the training set (See Appendix~\ref{app:dataset} for details). The concept instances on our images have varying locations, size, and relative positions, making them more difficult than fixed-location fixed-size inference tasks such as in \cite{shanahan2020explicitly}. We also provide 2D and 3D versions of images for demonstrating concept acquisition across domains.

\begin{table}[!t]
\centering
\caption{Performance of models on classification accuracy for hierarchical dataset 1 and 2 (\%) and on detection for hierarchical concepts with distractors. For the latter task, we use the pixel-wise intersection-over-union (IoU) (\%) as our metrics. The bold fonts in the tables indicate the best model comparing with baselines. The ``Statistics'' in classification  predicts the class that has the most global label fraction.
}
\resizebox{1\textwidth}{!}{%
\begin{tabular}{@{}lcccc} \toprule
& \multicolumn{2}{c}{\textbf{Classification (acc.)}} & \multicolumn{2}{c}{\textbf{Detection (IoU)}} \\
\textbf{Model} &  HD-Letter  & HD-Concept & \multicolumn{1}{c}{HD-Letter+distractor} & \multicolumn{1}{c}{HD-Concept+distractor} \\ \midrule
Statistics   & 46.5   & 53.5  &  5.69  &  12.6 \\
Heuristics   &  (--)   & (--)  &  42.3  &  29.2 \\
CADA-VAE \cite{schonfeld2019generalized}  & 18.0     &   66.0  &  (--)  &  (--) \\ \midrule
\textbf{\proj (ours)} &  \textbf{84.5}  &  \textbf{70.5} &  \textbf{72.5}   & \textbf{84.7}  \\ \midrule \midrule
\proj composition without R-EBM & 62.5  &   32.5 &  45.3  &   84.3  \\
\proj composition without HC-EBM & 67.0  &   55.0 & 67.7  &   78.4   \\
\proj without $L^\text{(pos-std)}$ &  43.6  &  65.5  & 76.1  &  81.5  \\
\proj without $L^\text{(neg)}$ &  64.5 &  59.0 & 60.0  &    84.2 \\
\proj without $L^\text{(em)}$ & 81.5   &  61.0  &  68.0  &   86.0  \\
\proj with only $L^\text{(Improved)}$ & 27.5   &  55.5 & 49.1   &  81.7   \\ \bottomrule
\end{tabular}}
\label{tab:classification-detection}
\vskip -0.1in
\end{table}

\subsection{Zero-shot Classification and Detection of Novel Concepts}
\label{sec:classify_detection}

To test recognition of novel concepts, we designed two datasets consisting of different concept and relation types. The HD-Letter hierarchical dataset consists of concept instances of ``line'' and relation instances of ``parallel'', ``perp-edge'' (perpendicular and touching edge), ``perp-mid'' (perpendicular and touching middle), together with distractor objects. Examples are provided in the form of 3-tuples $(image, mask, concept)$ for concepts, and $(image, mask_1, mask_2, relation)$ for relations. \proj models are trained via objective Eq. \ref{eq:objective} (See Appendix \ref{app:dataset} for examples of training datasets).
At inference time, the models need to perform classification and detection on novel images with more complicated English characters of ``E-shape'', ``F-shape'' and ``A-shape'', given their structure graphs with up to 4 nodes and 6 edges. For detection, the images also contain a few distractors (concepts unrelated to the ones the model is predicting). The model is asked to return a mask indicating which pixels belong to an instance of the specified novel concept. Although looking simple, this is actually a very challenging task. because in order to solve the problem of  classifying/detecting hierarchical concepts, the model needs to solve a subgraph isomorphism problem, which is NP-hard. For example, take the problem of detecting ``E'' shape in an image with distractor of ``T'' and ``Rectangle'' (Fig. \ref{fig:experiment_1} (a) first subfigure), a graph with 10 nodes (lines), 45 possible edges. Note that during training time, a model has only learned ``line'' concept, and relations  ``parallel'', ``perp-edge'' (perpendicular, touching edge), and ``perp-mid'', but not the overall ``E'' shape. At inference time and given the ``E'' concept graph  with 4 nodes (lines), 6 edges (relations), the model needs to find the ``E'' concept subgraph within the large image graph. This involves $C_{10}^{4} \times 4!=5040$ possible mask assignments. Additionally, a model may not perfectly detect the masks for low-level concepts.

For more complex concepts and relations, we designed the HD-Concept hierarchical dataset, which consists of training concepts of ``E-shape'' and ``rectangle'', and training relations of ``inside'', ``enclose'' and ``non-overlap''. The hierarchical concepts to be classified and detected are three characters which we term Concept1, Concept2 and Concept3, as the ground-truth masks in Fig. \ref{fig:experiment_1} (a) indicate. The three hierarchical concepts have the same multiset of concepts, but a different relation structures.

\begin{figure*}[t]
\centering
    \includegraphics[width=1.04\textwidth]{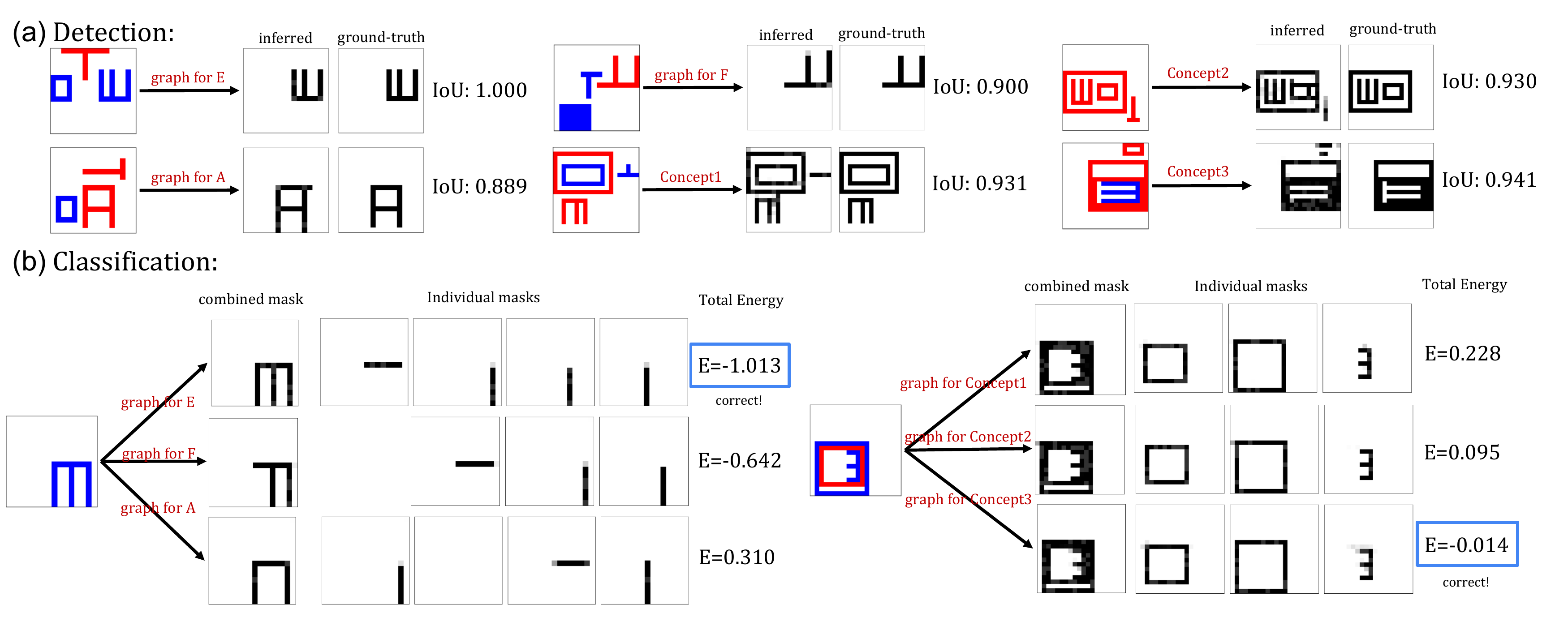}
    \vspace{-6mm}
    \caption{{\bf Sample of \proj's Inference Results.} Tasks are (a) pixel-wise detection, and (b) classification, for the datasets of HD-Letter (consisting of concept E-shape, F-shape and A-shape) and HD-Concept Dataset (consisting of Concept1, Concept2, Concept3). We see that \proj has high pixel-wise detection IoU for the specified concepts. For classification, \proj uses each candidate concept's  graph and performs SGLD on the composed HC-EBM to infer the mask and compute its energy. We see that \proj gives lower energy to the correct concept, even when sometimes candidate concepts have  similar inferred mask ((b) right).}
    \label{fig:experiment_1}
    \vspace{-5mm}
\end{figure*}

As the closest setting to our classification task is zero-shot classification,
we compare \proj to a state-of-the-art zero-shot learning algorithm CADA-VAE \cite{schonfeld2019generalized}, which we adapted to our setting, using the set of concepts and relations as feature embeddings to represent the concept graph (see Appendix \ref{app:cada_vae} for details). Additionally, we evaluated a ``Statistics'' baseline which samples random pixels on the image based on the global statistics of pixel occurrence, and a ``Heuristics'' baseline that randomly chooses a same-color connected object. We use a CNN-based architecture for \proj, as given in Appendix \ref{app:hicone_architecture}. We add comparisons with ablation of different aspects of \proj, including without HC-EBM terms or R-EBM terms in the composition rule of Def. \ref{def:composition_rule} (Eq. \ref{eq:composition}), and without different regularization terms of the objective in Eq. \ref{eq:objective}. 

Fig.~\ref{fig:experiment_1} shows a demonstration of this experiment.
From performance results in Table~\ref{tab:classification-detection}, we see that \proj achieves classification accuracy of 84.5\% and 70.5\% on HD-Letter and HD-Concept, respectively, both higher than CADA-VAE. The gap is more significant in HD-Letter, where there is a larger distribution shift for hierarchical concepts (as can be seen in  Appendix~\ref{app:dataset}), which the joint embedding of the image and features learned by CADA-VAE is insufficient to handle.
The reason for the low accuracy is that CADA-VAE is not able to address these out-of-domain distribution shifts (See Appendix \ref{app:cada_vae_explanation} for details). During inference, its embedding (multi-hot vector) for the graph structure can contain up to 10 hots (4 lines, 6 edges), while during training, it is only up to 1-hot. This example also shows the intrinsic difficulty of the task.
In addition, Table~\ref{tab:classification-detection} and Fig.~\ref{fig:experiment_1} (a) show that \proj is able to detect the hierarchical concepts in the presence of distractors, and that its performance is significantly better than ``Statistics'' and ``Heuristics'' baselines.

The results of our ablation studies in Table \ref{tab:classification-detection} show that both the HC-EBM and R-EBM terms in Eq.~\ref{eq:composition} are key to successful classification and detection. This is likely because without relations and using only concept terms, \proj may lose its ability to distinguish between different hierarchical concepts; for example, both ``E-shape'' and ``A-shape'' have 4 lines but different relation structures. Moreover, we see that both $L^\text{(pos-std)}$ and $L^\text{(neg)}$ improve classification and detection performance.

\begin{table}[t!]
\centering
\caption{Performance of models on acquiring  concepts between models and domains at inference time (\%).}
\resizebox{\textwidth}{!}{%
\begin{tabular}{@{}lcccc} \toprule
& \multicolumn{2}{c}{\textbf{Domain 1 (2D image) Parsing}} & \multicolumn{2}{c}{\textbf{Domain 2 (3D image)}} \\
\textbf{Model} & Isomorphism (acc.) $\uparrow$  & Edit distance $\downarrow$  & Classification (acc.) $\uparrow$ & Detection (IoU) $\uparrow$ \\ \midrule
Statistics   & 2.33   & 3.14  & 33.3  &  2.53 \\
Mask R-CNN \cite{he2017mask}+relation classification   &  35.5 & 1.01 &  (--)   &  (--) \\ \midrule
\textbf{ZeroC$_1\to$\ ZeroC$_2$ (ours)} &  \textbf{72.7}  & \textbf{0.50} &   \textbf{60.7}   & \textbf{94.4} \\ \midrule \midrule
ZeroC$_1$ without $L^\text{(pos-std)}\to$\ ZeroC$_2$  & 55.2  & 1.57  &   54.7 &   90.5 \\
ZeroC$_1$ without $L^\text{(neg)}\to$\ ZeroC$_2$  & 53.5 & 0.99 &  52.5  &  92.1  \\ 
ZeroC$_1$ without $L^\text{(em)}\to$\ ZeroC$_2$  & 50.7 & 1.58 &  51.3  &  95.4  \\
ZeroC$_1$ with only $L^\text{(Improved)}\to$\ ZeroC$_2$  &  11.5 &  2.00 &  50.5  &   94.7  \\%%
ZeroC$_2$ with ground-truth graph (upper bound)  & (--) & (--) &  61.8  & 94.2 \\ \bottomrule
\end{tabular}}
\label{tab:transfer}
\vspace{-4mm}
\end{table}

\begin{figure}[t]
\centering
    \includegraphics[scale=0.65]{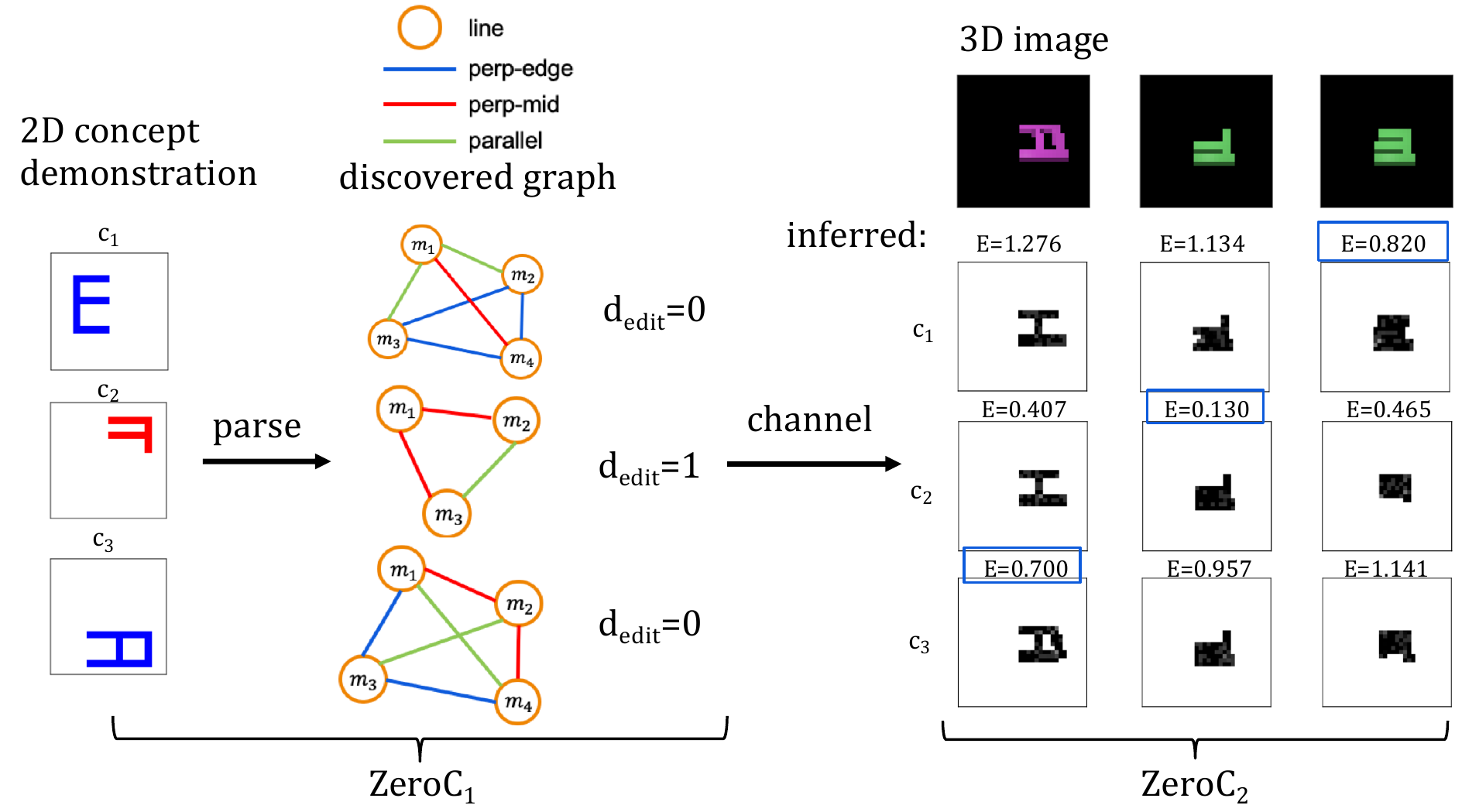}
    \caption{{\bf Acquiring Hierarchical Concept between Domains.} The figure shows actual example tasks and results. At inference, ZeroC$_1$ sees the 2D demonstration of three images showing three unseen concepts $c_1$, $c_2$, $c_3$. It first parses each image into respective concept graphs. We see that except for $c_2$ that has an edit distance $d_\text{edit}$ of 1, the others have perfect parsing. ZeroC$_1$ then sends to ZeroC$_2$ the parsed concept graphs, which ZeroC$_2$ uses to perform classification that selects which 3D image corresponds to each concept.}
    \label{fig:transfer-exp}
    \vspace{-7mm}
\end{figure}

\subsection{Acquiring Novel Hierarchical Concepts Across Domains}
\label{sec:acquiring_exp}

We show that \proj can acquire novel hierarchical concepts across models and domains (Fig. \ref{fig:transfer-exp}). We train a \proj model in domain ZeroC$_1$, where the images are 2D and one-hot color-coded, containing the the same training concepts and relation types as in HD-Letter in Sec. \ref{sec:classify_detection}. \emph{Independently}, we train another model in a different domain ZeroC$_2$, where the images contain the same set of concepts and relations, but are viewed in 3D from a certain camera angle, have larger size, and use RGB colors. At inference time, each test task consists of a tuple of three images showing the hierarchical concepts in the first domain to ZeroC$_1$. ZeroC$_1$ is only allowed to send symbolic information, up to a few bits, to ZeroC$_2$. Then ZeroC$_2$ performs classification and detection tasks on three example images in the second domain. In addition, we also evaluate whether ZeroC$_1$ can parse the concept graph of the hierarchical concept correctly, using the metric of graph edit distance $d_\text{edit}$ and graph-isomorphism accuracy w.r.t. the ground-truth. Note that the  graph-isomorphism accuracy is a stringent metric, which is only 1 if the parsed graph is isomorphic to the ground-truth, and 0 otherwise\footnote{This is a stringent metric since, e.g. for ``Eshape'' that contains 4 concept nodes and 6 relation edges, an individual accuracy of 0.9 would result in $\sim 0.9^{4+6}=0.349$ isomorphism accuracy, and an individual accuracy of 0.8 would result in $\sim 0.8^{4+6}=0.107$ isomorphism accuracy.}. 
We compare with a strong baseline of Mask R-CNN \cite{he2017mask} for object detection together with relation classification. The relation classification network is a CNN which takes as input a pair of embeddings produced by the Mask R-CNN, and predicts the label of the relation (see Appendix \ref{app:mask_rcnn} for details). 
We did extensive tuning of the Mask R-CNN approach to obtain its best performance.
We also evaluate the performance of ZeroC$_2$ given the ground-truth structure graph, as an ideal scenario of perfect parsing, providing an upper bound for the performance.

Table \ref{tab:transfer} shows that ZeroC$_1$ achieves 72.7\% graph isomorphism accuracy for parsing the hierarchical concepts in the first domain, significantly higher than Mask R-CNN + relation classification. After sending the concept graphs in the second domain, ZeroC$_2$ achieves a classification accuracy of 60.7\% and detection IoU of 94.4\%. Without this information and relying on global statistics, the classification accuracy of ZeroC$_2$ is 33.3\%. This demonstrates that ZeroC$_2$ is able to acquire novel hierarchical concepts at \emph{inference} time, from an \emph{independently} trained model from a \emph{different domain}. The ablation study shows that $L^\text{(em)}$, $L^\text{(neg)}$ and $L^\text{(pos)}$ all contribute to more accurate parsing. The standard technique of EBM training \cite{du2020improved}
are insufficient to achieve good parsing, leading to less accurate transfer of hierarchical concepts. We also see in the ablations that even without perfect parsing from ZeroC$_1$, the reduction of classification accuracy for ZeroC$_2$ is small, showing that it is able to classify under noisy specification of the concept graph. We discuss further on the generality (Appendix \ref{app:generality}), scalability (Appendix \ref{app:scalability}) and computational complexity (Appendix \ref{app:complexity}) in respective Appendix sections.

\vspace{-1mm}
\section{Related Work}
\vspace{-1mm}

Our work relates to visual compositionality, concept and relation learning, and zero-shot learning.

\xhdrx{Visual compositionality}
Compositionality is a key for addressing diverse tasks given finite basic knowledge.
Some approaches use composition EBMs for generation~\cite{du2020compositional}, and a VAE-based architecture for bi-directional symbol-image generation that can also learn logical recombination of concepts~\cite{higgins2017scan}. These two works focus on composing \emph{factors of variation}, \textit{e.g.} color, position, smiling, young. In contrast, our work focuses on concepts that abstract objects, where the internal hierarchical relational structure is key. Moreover, while the above works focus on generation, our work focuses on the tasks of classification and detection. A novel Bayesian-based method for few-shot learning on the Omniglot dataset shows that compositionality is pivotal for improved performance~\cite{lake2015human}. They achieve compositionality via hierarchical MCMC sampling on hand-coded priors of elementary concepts and relations. In comparison, our work is neural network-based and only requires demonstrations of elementary concepts and relations, reducing hand-coded priors. Another approach introduced a modular neural network, which uses composition of neural modules for visual question answering~\cite{andreas2016neural}. Their method can be seen as composing transformations on a representation. In comparison, our composition is achieved via composition of energy landscapes. While they focus on question answering on scene graphs, we focus on classification and detection of hierarchical concepts.

\xhdrx{Concept and relation learning}
There has been exciting progress in concept learning and relation learning. Works in concept learning generally represent concepts in latent space via prototypes \cite{snell2017prototypical,cao2020concept}, or via latent embedding such as SCAN \cite{higgins2017scan}, InfoGAN \cite{chen2016infogan} and Neuro-Symbolic Concept Learner \cite{mao2019neuro}. \cite{mordatch2018concept} introduced EBMs to represent concepts with a demonstration in simple state space. \cite{du2021unsupervised} further introduced unsupervised learning of local and global concepts with EBMs. Regarding relation learning, non-local neural networks \cite{wang2018non}, Relation Networks  \cite{santoro2017simple}, Neural Relational Inference \cite{kipf2018neural} and C-SWM \cite{Kipf2020Contrastive} use latent complete graphs as inductive biases to 
represent potential relations. PrediNet \cite{shanahan2020explicitly} explicitly represents propositions, relations, and objects with a transformer-based architecture, and demonstrates that it can learn relations that generalize to novel shapes of objects and column patterns. In comparison to the above works, \proj explicitly learns both concepts and relations, which has the unique capability to recognize and acquire hierarchical concepts at inference.

\xhdrx{Zero-shot learning}
Zero-shot learning methods \cite{romera2015embarrassingly,jiang2017learning,bucher2017generating,zhao2018zero,schonfeld2019generalized} typically learn a joint embedding space between image and feature labels, and perform classification at inference time on images with novel classes, based on how the novel classes correspond to a set of features. Our generalization task also requires generalizing to new concepts without seeing the image (zero-shot), but has the important distinctions that at inference time, our concepts to be inferred lie at a \emph{higher} hierarchy than that in training, and furthermore use the structural concept graph. In comparison, in standard zero-shot learning, concepts in training and inference lie at the \emph{same} hierarchy level, and only generalize to new combinations of features (constituent concepts) while neglecting relation structures.

\vspace{-2mm}
\section{Conclusion}
\vspace{-1mm}
\label{sec:discussion}

In this paper, we introduce \proj, a new framework for zero-shot concept recognition and acquisition at inference time. Our experiments show that in a challenging grid-world domain, \proj is able to recognize complex, hierarchical concepts composed of English characters in a grid-world in a zero-shot manner, being given a high-level, symbolic specification of their structures, and after being trained with simpler concepts. In addition, we demonstrate that an independently trained \proj is able to transfer hierarchical concepts across different domains at inference. Although this work is evaluated only in grid-world domain, we are the first to address this difficult challenge, and hope that this work will make a useful step in the development of composable neural systems, capable of zero-shot concept recognition and acquisition and hence suitable for more diverse tasks.

\ack{}
We thank Rui Yan and Blaž Škrlj for discussions and for providing feedback on our manuscript.
We also gratefully acknowledge the support of
DARPA under Nos. HR00112190039 (TAMI), N660011924033 (MCS);
ARO under Nos. W911NF-16-1-0342 (MURI), W911NF-16-1-0171 (DURIP);
NSF under Nos. OAC-1835598 (CINES), OAC-1934578 (HDR), CCF-1918940 (Expeditions), 
NIH under No. 3U54HG010426-04S1 (HuBMAP),
Stanford Data Science Initiative, 
Wu Tsai Neurosciences Institute,
Amazon, Docomo, GSK, Hitachi, Intel, JPMorgan Chase, Juniper Networks, KDDI, NEC, and Toshiba.

The content is solely the responsibility of the authors and does not necessarily represent the official views of the funding entities.

\bibliography{reference}
\bibliographystyle{IEEEtran}

\section*{Checklist}

\begin{enumerate}

\item For all authors...
\begin{enumerate}
  \item Do the main claims made in the abstract and introduction accurately reflect the paper's contributions and scope?
    \answerYes{}
  \item Did you describe the limitations of your work?
    \answerYes{In conclusion and Appendix \ref{app:limitations}.}
  \item Did you discuss any potential negative societal impacts of your work?
    \answerYes{In Appendix \ref{app:social_impact}.}
  \item Have you read the ethics review guidelines and ensured that your paper conforms to them?
    \answerYes{}
\end{enumerate}

\item If you are including theoretical results...
\begin{enumerate}
  \item Did you state the full set of assumptions of all theoretical results?
    \answerYes{}
        \item Did you include complete proofs of all theoretical results?
    \answerYes{See Appendix \ref{app:empty_mask_proof}.}
\end{enumerate}

\item If you ran experiments...
\begin{enumerate}
  \item Did you include the code, data, and instructions needed to reproduce the main experimental results (either in the supplemental material or as a URL)?
     \answerYes{The Appendix includes full details on model architecture, training and evaluation to reproduce the experimental results. Code and data will be released upon publication of the paper.}
  \item Did you specify all the training details (\textit{e.g.}, data splits, hyperparameters, how they were chosen)?
   \answerYes{Full details to reproduce the experiments are included in the Appendix.}
        \item Did you report error bars (\textit{e.g.}, with respect to the random seed after running experiments multiple times)?
    \answerNo{}
        \item Did you include the total amount of compute and the type of resources used (\textit{e.g.}, type of GPUs, internal cluster, or cloud provider)?
    \answerYes{See Appendix of respective experiment section.}
\end{enumerate}

\item If you are using existing assets (\textit{e.g.}, code, data, models) or curating/releasing new assets...
\begin{enumerate}
  \item If your work uses existing assets, did you cite the creators?
    \answerNA{We use our own created benchmark.}
  \item Did you mention the license of the assets?
    \answerNA{}
  \item Did you include any new assets either in the supplemental material or as a URL?
    \answerYes{In Appendix \ref{app:dataset}, we describe the engine and training/testing tasks. Full dataset will be released upon publication.}
  \item Did you discuss whether and how consent was obtained from people whose data you're using/curating?
    \answerNA{}
  \item Did you discuss whether the data you are using/curating contains personally identifiable information or offensive content?
    \answerNA{}
\end{enumerate}

\item If you used crowdsourcing or conducted research with human subjects...
\begin{enumerate}
  \item Did you include the full text of instructions given to participants and screenshots, if applicable?
    \answerNA{}
  \item Did you describe any potential participant risks, with links to Institutional Review Board (IRB) approvals, if applicable?
    \answerNA{}
  \item Did you include the estimated hourly wage paid to participants and the total amount spent on participant compensation?
    \answerNA{}
\end{enumerate}

\end{enumerate}

\appendix

\newpage
\section{Appendix}

In the Appendix, we provide details that complement the main text. Specifically, in Appendix \ref{app:learning} and \ref{app:regularizations}, we detail the learning algorithm of HC-EBMs and the learning objective these models optimize. In Appendix \ref{app:empty_mask_proof}, we prove a necessary conditions for \proj to correctly discover the underlying concept graph. In Appendix \ref{app:hicone_architecture}, we provide the exact neural architecture used to parameterize HC-EBMs. In Appendix \ref{app:acquire}, we present further explanation on how hierarchical concepts acquired in one domain may be transferred to different domain, as well as a note on time complexity of inference. In Appendix \ref{app:cada_vae} and \ref{app:cada_vae_explanation}, we present implementation details on the CADA-VAE baseline for classifying hierarchical concepts and an analysis of its limited performance. In Appendix \ref{app:mask_rcnn}, we similarly present implementation details on the Mask R-CNN + Relation Classification baseline for acquiring concepts. In Appendix \ref{app:dataset}, we detail the generation process of our 2D and 3D grid-world datasets. In Appendix \ref{app:limitations}, we explain some limitations of our current work and propose several future directions to explore. In Appendix \ref{app:social_impact}, we discuss our work's broader social impact. In Appendix \ref{app:additional_exp}, we present additional experimental results on a CLEVR dataset to demonstrate the potential of our framework to generalize to more real-world settings. In Appendix \ref{app:generality}, we further discuss the generality of our framework to different sets of elementary concepts and relations and to other datasets. In Appendix \ref{app:scalability}, we explain the scalability of the framework with respect to task complexity, inference time complexity, and image complexity. In Appendix \ref{app:complexity}, we summarize ZeroC's computational complexity and our empirical observations.

\subsection{Learning algorithm}
\label{app:learning}

Here we give the learning algorithm for \cebms, which can be elementary or hierarchical. A similar algorithm applies to \rebms, by simply replacing the $m$ by $m_1, m_2$. Here we omit the subscript of $X,M,C$ for clarity. 

In our experiments, we perform hyperparameter search over the coefficients, to optimize the average classification accuracy and detection IoU on the validation set that has the same type of concepts/relations (not validating on the inference tasks in Section \ref{sec:experiments}). For training HC-EBMs, we use $\alpha_\text{pos-std}=0.1, \alpha_\text{neg}=0.05$ and $\alpha_\text{em}=0.1$. For training R-EBMs, we use $\alpha_\text{pos-std}=1, \alpha_\text{neg}=0.2$ and $\alpha_\text{em}=0$. We use learning rate $10^{-4}$, number of SGLD steps $K=60$ for training and $K=150$ for inference. During inference, we use an ensemble size of 64 to perform MAP estimation as in Eq. (\ref{eq:classify}), and use ensemble size of 256 for detection.

\begin{algorithm}[h]
   \caption{\textbf{Algorithm for learning \cebms}}
\label{alg:learning}
\begin{algorithmic}
   \STATE {\bfseries Require:} data dist. $p_D(\x,m,c)$, \cebm $E_\theta$, 
   \STATE {\bfseries Require:} step size $\lambda$, number of steps $K$, random mask generator $\mathcal{U}_a$, random embedding generator $\mathcal{U}_c$, coefficients $\alpha_\text{pos-std}$, $\alpha_\text{neg}$, and $\alpha_\text{em}$.
\STATE 1: \ \ $\mathcal{B}\gets \emptyset{}$
\STATE 2:\ \ \  \textbf{while} not converged \textbf{do}
\STATE 3:\ \ \ \ \ \ \ \ \  $(\x_n^+, m_n^+, c_n^+)\sim p_D$
\STATE \ \ \ \ \ \ \ \ \  \ \ \  \textit{// Generate samples from Langevin dynamics:}
\STATE 4:\ \ \ \ \ \ \ \ \  $\text{rand}\sim U[0,1]$
\STATE 5:\ \ \ \ \ \ \ \ \  \textbf{if} $\text{rand}\in[0,1/4)$ \textbf{do}
\STATE 6:\ \ \ \ \ \ \ \ \ \ \ \ \ \ \ \ \ $(\x_n, \tm_n^0, c_n^+)\sim \mathcal{B}$ with 20\% probability and $\tm_n^0\sim\mathcal{U}_a$ otherwise
\STATE 7: \ \ \ \ \ \ \ \ \ \ \ \ \ \ \ \  $\tm_n^K\gets\textbf{SGLD}_{\tm}(E_\theta; \x_n, \tm_n^0, c_n^+)$
\STATE 8:\ \ \ \ \ \ \ \ \ \ \ \ \ \ \ \ \  $(\x_n^-,m_n^-,c_n^-)\gets(\x_n,\tm_n^K, c_n^+)$
\STATE 9:\ \ \ \ \ \ \ \ \  \textbf{elseif} $\text{rand}\in[1/4,1/2)$ \textbf{do}
\STATE 10:\ \ \ \ \ \ \ \ \ \ \ \ \ \ \ $(\x_n, m_n^+, \tc_n^0)\sim \mathcal{B}$ with 20\% probability and $\tc_n^0\sim\mathcal{U}_c$ otherwise
\STATE 11: \ \ \ \ \ \ \ \ \ \ \ \ \ \  $\tc_n^K\gets\textbf{SGLD}_{\tc}(E_\theta; \x_n, m_n, \tc_n^0)$
\STATE 12:\ \ \ \ \ \ \ \ \ \ \ \ \ \ \  $(\x_n^-,m_n^-,c_n^-)\gets(\x_n,m_n^+, \tc_n^K)$
\STATE 13:\ \ \ \ \ \ \  \textbf{elseif} $\text{rand}\in[1/2,3/4)$ \textbf{do}
\STATE 14:\ \ \ \ \ \ \ \ \ \ \ \ \ \ \ $(\x_n, \tm_n^0, \tc_n^0)\sim \mathcal{B}$ with 20\% probability and $\tm_n^0\sim\mathcal{U}_m,\tc_n^0\sim\mathcal{U}_c$ otherwise
\STATE 15: \ \ \ \ \ \ \ \ \ \ \ \ \ \  $\tm_n^K, \tc_n^K\gets\textbf{SGLD}_{\tm,\tc}(E_\theta; \x_n, \tm_n^0, \tc_n^0)$
\STATE 16:\ \ \ \ \ \ \ \ \ \ \ \ \ \ \  $(\x_n^-,m_n^-,c_n^-)\gets(\x_n,\tm_n^K, \tc_n^K)$
\STATE 17:\ \ \ \ \ \ \  \textbf{else do}
\STATE 18:\ \ \ \ \ \ \ \ \ \ \ \ \ \ \ $(\tx_n, \tm_n^0, \tc_n^0)\sim \mathcal{B}$ with 20\% probability and $\tx_n^0\sim\mathcal{U}_x$ otherwise
\STATE 19: \ \ \ \ \ \ \ \ \ \ \ \ \ \  $\tx_n^K\gets\textbf{SGLD}_{\tx}(E_\theta; \tx_n^0, m_n, c_n)$
\STATE 20:\ \ \ \ \ \ \ \ \ \ \ \ \ \ \  $(\tx_n^-,m_n^-,c_n^-)\gets(\tx_n^K,m_n, c_n)$
\STATE 21:\ \ \ \ \ \ \  \textbf{end if}

\STATE \ \ \ \ \ \ \ \ \  \ \ \  \textit{// Optimize objective for $E_\theta$ wrt $\theta$ with Eq. \ref{eq:objective}:}
\STATE 22: \ \ \ \ \ \ \  $\Delta \theta \gets \nabla_\theta  \frac{1}{N}\sum_n\left(L^\text{(Improved)}_n +\alpha_\text{pos-std}L_n^\text{(pos-std)}+\alpha_\text{em}L^\text{(em)}_n+\alpha_\text{neg}L^\text{(neg)}_n\right)$
\STATE 23: \ \ \ \ \ \ \ Update $\theta$ based on $\Delta \theta $ using Adam optimizer
\STATE 24:  \ \ \ \ \ \ \ $\mathcal{B}\gets\mathcal{B}\cup (\x_n^-,m_n^-,c_n^-)$
\STATE 25: \textbf{end while}
\STATE 26: \textbf{return} $E_\theta$ 
\end{algorithmic}
\end{algorithm}

\begin{algorithm}[t]
   \caption{\textbf{Stochastic Gradient Langevin Dynamics (SGLD)}}
\label{alg:SGLD}
\begin{algorithmic}
   \STATE {\bfseries Require:} energy-based model $E_\theta$ with concept (relation) embedding $c$ ($r$)
\STATE {\bfseries Require:} SGLD target $\tq$, choose from $\tm$, $\tc$,  $(\tm,\tc)$  or $x$
\STATE {\bfseries Require:} Input $\x$
\STATE {\bfseries Require:} step size $\lambda$, number of steps $K$
\STATE 1: \textbf{for} $k=1$ to $K$ \textbf{do}
\STATE 2: \ \ \ \ \ \  $\tq^k\gets \tq^{k-1} - \frac{\lambda}{2}\nabla_q E_\theta(\x;\tq^{k-1})+\epsilon^{k-1}$,  where $\epsilon^{k}\sim N(0,\sigma^2)$, $\sigma^2=\lambda$
\STATE 3: \textbf{end for}
\STATE 4: \textbf{return} $\tilde{q}^K$ 
\end{algorithmic}
\end{algorithm}

\begin{figure}[h!]
\centering
     \includegraphics[width=0.85\textwidth]{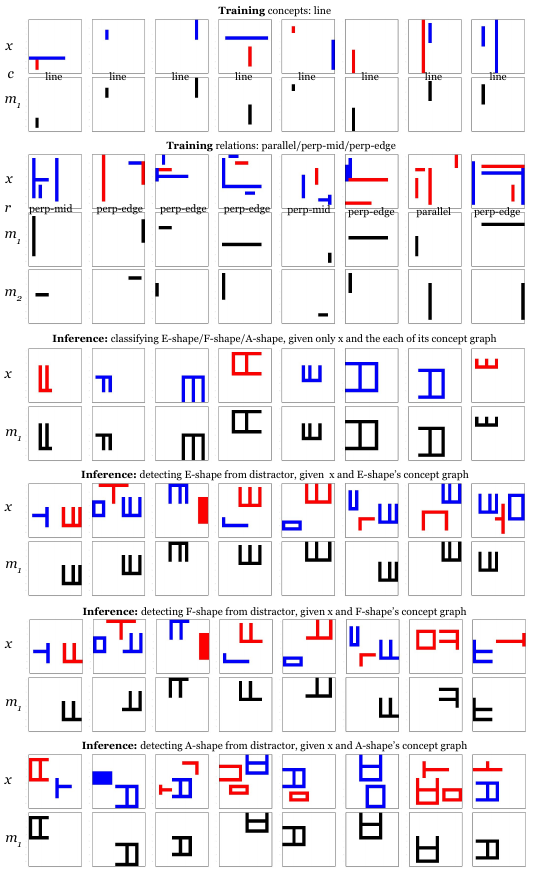}
      \caption{Samples examples from the HD-Letter dataset for training and inference of zero-shot concept classification and detection. More detail in Appendix \ref{app:dataset}. The models are trained with concept of ``line'' and relations of ``parallel'', ``perp-mid'', ``perp-edge''. At inference, the models are tasked to perform classification and pixel-wise detection, on hierarchical concepts (w.r.t. trainnig) of ``E-shape'', ``F-shape'' and ``A-shape''. We see that the concepts in inference is more complex than those in training. \textit{E.g.} for detecting ``Eshape'' in inference, a model will need to compose up to 4 nodes (of ``line'' concepts) and 6 edges of ``parallel'', ``perp-mid'', ``perp-edge'' relations.}
       \label{fig:hc-letter-examples}
\end{figure}

\subsection{Learning objective}
\label{app:regularizations}

Here we detail the learning objective for \proj and provide explanations, complementing Sec. \ref{sec:learning}. The $L_n^\text{improved}$ is the objective proposed in \cite{du2020improved}, neglecting the entropy regularization term. In addition, as explained in Sec. \ref{sec:learning}, we introduce three regularizations to inject the right inductive biases to address the tasks of zero-shot concept recognition and acquisition:

\begin{align}
\label{eq:objective_appendix}
\left\{
\begin{array}{ll}
L_n^\text{(pos-std)}=\left(\text{Var}_n(E_\theta(x_n^+,m_n^+,c_n^+))\right)^\frac{1}{2}\\
L^\text{(em)}_n=\left|\frac{E_\theta(\x_n^+,m_n^+,c_n^+)+E_\theta(\x_n^+,m_n^-,c_n^+)}{2}-E_\theta(\x_n^+,m_n^\text{(em)},c_n^+)\right|\\
L^\text{(neg)}_n=-E_\theta(x_n^+,m_n^\text{(neg)},c_n^+)
\end{array}
\right.
\end{align}

The superscript ``$+$'' denotes positive examples from ground-truth, ``$-$'' denotes negative examples generated from SGLD, ``(em)'' denotes an empty mask, and ``(neg)'' denotes algorithmically generated negative examples. 
The additional regularizations are $L_n^\text{(pos-std)}$ for reducing energy variance for positive examples, $L_n^\text{(em)}$ for empty-mask regularization, and $L_n^\text{(neg)}$ which is algorithmically generated negative examples. Note that unlike standard image-only EBMs, the training needs to take multiple modalities into account: image $x$, mask $m$ and concept $c$. We perform conditional SGLD w.r.t. each modality to generate negative examples. \textit{E.g.}, when generating negative examples for masks $m_n^-$, we provide ground-truth of image and concept label $x_n^-\gets x_n^+$, $c_n^-\gets c_n^+$. Thus, $(x_n^-,m_n^-,c_n^-)=(x_n^+,m_n^-,c_n^+)$ is a negative tuple even though  $x_n^-$ and $c_n^-$ are from ground-truth. For each minibatch we randomly sample the modality the conditional SGLD is performed on (steps 4-21 in Alg. \ref{alg:learning}).

\xhdrx{Reducing variance of energy for positive examples}
The use of only contrastive divergence in typical EBM training, \textit{i.e.} pushing down energy for positive examples and pushing up energy for negative examples, is insufficient, since the composed \cebm needs to discover the masks of \emph{all} its constituent models. 
For example, with 2 objects in the image $x^+$ with concepts $c_1^+$ and $c_2^+$, respectively, we can use $E(x^+,m_1,c^+_1)+E(x^+,m_2,c^+_2)$ according to our composition rule (Def. \ref{def:composition_rule}) to detect their respective masks. However, $L^\text{(improved)}$ only encourages $(x^+,m^+_1,c^+_1)$ to be lower than $(x^+,m^-_1,c^+_1)$ locally, but it can still be higher than $(x^+,m^-_2,c^+_2)$ for a negative mask $m_2^-$ for concept $c_2^+$. Then the composed energy model will favor discovering $m_2^-$ instead of $m_1^+$. Thus, we add $L^\text{(pos-std)}$ to encourage similar energy for positive examples, thus lower than the energy for negative examples.

\xhdrx{Empty-mask regularization} To properly perform parsing (Alg. \ref{alg:parsing}) when there are more energy terms than actual instances, we require that the redundant masks are empty instead of being some random negative masks. in Appendix \ref{app:empty_mask_proof}, we prove that the necessary condition for Alg. \ref{alg:parsing} to correctly discover the underlying concept graph is that the energy of the empty mask lies between the energy of positive examples and negative examples. Intuitively,  given the image $x_n^+$ and concept $c_n^+$, the energy $E(x_n^+,m_n^\text{(em)},c_n^+)$ with empty mask $m_n^\text{(em)}$ should be between the positive energy $E(x_n^+,m_n^+,c_n^+)$ and negative energy $E(x_n^+,m_n^-,c_n^+)$ for the empty mask to appear before random negative masks appear. Therefore, in $L_n^\text{(em)}$, we encourage $E(x_n^+,m_n^\text{(em)},c_n^+)$ to be near the average of positive and negative energy. If there is redundant energy terms during parsing, the corresponding masks will favorably become all-zero instead of some random negative mask.

\xhdrx{Algorithmically generated negative examples} To encourage each mask to discover a single concept instance instead of a partial instance, we randomly corrupt the tuple $(x_n^+, m_n^+, c_n^+)$ and push up the corresponding energy.

\xhdrx{Reason to neglect the entropy regularization} The entropy term in \cite{du2020improved} serves to increase the diversity of the generated examples. And the computation of entropy requires many examples. This is fine in \cite{du2020improved} since the EBM there has the form of E(x) which only needs to generate images \emph{unconditionally}, and the entropy can be estimated using all previous generated images x. In our work, our EBM are $E(x,m,c)$ and $E(x,m_1,m_2,c)$, and we need to generate the mask \emph{conditionally}, e.g. generate mask m conditioned on the image $x$ and label $c$. The entropy term would need to be a conditional entropy of $m$ given $x$ and $c$, where the pool of mask $m$ should be different for each individual image $x$ and label $c$. This requires, e.g. for each $x$, $c$, we generate over 100 masks to estimate the entropy which is computationally expensive, while currently we only need to sample 1 mask. Moreover, typically there are limited correct masks for a concept in an image, and encouraging diversity may not help the model identify the correct mask. In fact, we have tried empirically with keeping the entropy term and it results in a much worse accuracy, likely due to the above reason.

\subsection{Proof for necessary condition for correctly parsing the graph}
\label{app:empty_mask_proof}

Here we provide the proof that the energy for an empty mask needs to lie between the positive energy and negative energy, justifying the introduction of $L_n^\text{(em)}$ in Eq. (\ref{eq:objective}). Specifically, we prove:

\begin{theorem}
Let $E^+=E_\theta(x_n^+,m_n^+,c_n^+)$ be the ``positive energy'' for all positive examples of $(x_n^+,m_n^+,c_n^+)$, and $E^-=E_\theta(x_n^-,m_n^-,c_n^-)$ be the ``negative energy'' for any negative examples\footnote{Here we make the simplifying assumption that all positive energies have the same value, and all negative energies have the same value. In fact, the $L_n^\text{(pos-std)}$ encourages that the positive energies to be similar, to be able to discover all relevant concepts.} of $(x_n^+,m_n^+,c_n^+)$. Let $E^\text{(em)}$ be the energy for an $(x_n^+,m^{(0)},c_n^+)$ where $m^{(0)}=\mathbf{0}$ is an empty mask. A necessary condition for correctly discovering the underlying concept graph with Alg. \ref{alg:parsing} is that
\begin{equation}
E^+<E^\text{(em)}<\text{min}(E^-, E^+ + E^\text{(overlap)})
\end{equation}

Here $E^\text{(overlap)}$ is the energy added in Alg. \ref{alg:parsing} to penalize the concept EBMs to discover overlapping concepts.
\end{theorem}

\begin{proof}
Suppose that in the image there are in total $m\ge1$ objects, and there are in total $n\ge1$ concept EBMs. We want that only $m$ EBMs have their masks enabled and all the rest $n-m$ masks are empty (if $n\ge m$), \textit{i.e.} this configuration should have the lowest total energy. In other words suppose that instead there are $k$ masks that finds the objects (if will not overlap until $k>m$, and $n-k$ remains empty or have negative masks, then it should have a higher energy:

\begin{align}
&m \cdot E^+ + \lfloor n-m \rfloor \cdot E^\text{(em)}  \le k\cdot E^+ + \lfloor k-m \rfloor \cdot E^\text{(overlap)} + (n-k)\cdot E^\text{(em)}\label{eq:cond_1}\\
&m \cdot E^+ + \lfloor n-m \rfloor \cdot E^\text{(em)}  < k\cdot E^+ + \lfloor k-m \rfloor \cdot E^\text{(overlap)} + (n-k)\cdot E^-\label{eq:cond_2}
\end{align}

These two expressions should hold for any $k,m$.

\pagebreak
Setting $k=m$ in Eq. (\ref{eq:cond_2}), we have

$$E^\text{(em)}<E^-$$

From Eq. (\ref{eq:cond_1}), if $n\ge k > m$, and after re-arranging terms, we have

$$E^\text{(em)}<E^+ + E^\text{(overlap)}$$

From Eq. (\ref{eq:cond_1}), if $n\ge m > k$, and after re-arranging terms, we have

$$E^+<E^\text{(em)}$$

Combining the above three conditions, we have 

\begin{equation}
E^+<E^\text{(em)}<\text{min}(E^-, E^++E^\text{(overlap)})
\end{equation}
which concludes the proof.

This justifies the $L_n^\text{(em)}=\left|\frac{E_\theta(\x_n^+,m_n^+,c_n^+)+E_\theta(\x_n^+,m_n^-,c_n^+)}{2}-E_\theta(\x_n^+,m_n^\text{(em)},c_n^+)\right|=|\frac{1}{2}(E^++E^-)-E^\text{(em)}|$ in Eq. (\ref{eq:objective}), where it encourages that the $E^\text{(em)}$ stays between $E^+$ and $E^-$, and penalizes the deviation of $E^\text{(em)}$ to $\frac{1}{2}(E^++E^-)$.

\end{proof}

\subsection{Network architecture of \proj}
\label{app:hicone_architecture}

For all experiments in the paper, we use the \emph{same} architecture of concept and relation EBMs, with the only difference being the number of input channels (10 for 2D images and 3 for 3D images).
We provide in Table \ref{tab:cebm} the architecture of a \cebm, which consists of several ResBlocks (Table \ref{tab:cres}). Adding the residual to the final output is denoted as Skip(). When downsampling is performed, the residual is the output of two fully-connected layers applied to a flattened input image; otherwise, the residual is the input image. Chunk() splits an input vector into two equal-sized halves and expands both halves along a 2d grid with the same dimensions as the input image. The C\_Embed() architecture is detailed in Table \ref{tab:cembed}, with c\_dim denoting the dimension of the concept embedding.

\begin{minipage}[t]{0.48\linewidth}\centering
    \captionof{table}{ResBlock($x, c$) Architecture}
    \begin{tabular}{c}
        \hline
         Type \\
         \hline
         c\_embed\_1, c\_embed\_2 $\leftarrow$ Chunk(C\_Embed($c$)) \\
         Concat($x$, c\_embed\_1) \\
        $3 \times 3$ Conv2d, 64, Spectral Norm \\ 
         Activation() \\
         Concat(out, c\_embed\_2) \\ 
         $3 \times 3$  Conv2d, 64, Spectral Norm \\
         Activation() \\
         Skip() \\
         \hline
    \end{tabular}\label{tab:cres}
\end{minipage}
\begin{minipage}[t]{0.48\linewidth}\centering
    \captionof{table}{C\_Embed Architecture}
    \begin{tabular}{c}
        \hline
         Type \\
         \hline
         Dense(c\_dim, $4 \cdot$ c\_dim) \\
         LeakyRelu(0.2) \\
         Dense($4 \cdot$ c\_dim, $4 \cdot$ c\_dim) \\
         LeakyRelu(0.2) \\
         Dense($4 \cdot$ c\_dim, 4) \\
         \hline
    \end{tabular}\label{tab:cembed}
\end{minipage}

\begin{table}[ht]
    \begin{center}
    \caption{\cebm $E_{X, M, C}(x, m, c)$ Architecture} 
    \begin{tabular}{c|c}
        \hline
         Type, \# Channels & Activation \\
         \hline
         Concat($x, m$) & (-) \\
         $3 \times 3$ Conv2d, 128 & LeakyRelu(0.01) \\  
         ResBlock (Downsample), 128 &  LeakyRelu(0.01) \\
         ResBlock, 128 & LeakyRelu(0.01) \\
         ResBlock (Downsample), 256 & LeakyRelu(0.01)\\
         ResBlock, 256 & LeakyRelu(0.01) \\
         ResBlock (Downsample), 256 & LeakyRelu(0.01)\\
         ResBlock, 256 & LeakyRelu(0.01) \\
         Global Average Pooling & (-)  \\
         Dense() $\to$ 1 & (-) \\
         \hline
    \end{tabular}\label{tab:cebm}
    \end{center}
\end{table}

\subsection{Acquiring Hierarchical Concepts}
\label{app:acquire}

We present details about the algorithm for acquiring hierarchical concepts across models and domains. 
The algorithm works as follows (see also Alg.~\ref{alg:communication})

\begin{algorithm}[t]
\caption{\textbf{Acquiring Hierarchical Concepts Between Models and Domains}}
\label{alg:communication}
\begin{algorithmic}
\STATE {\bfseries Require:} HC-EBM$_1$ and R-EBM$_1$ in domain 1, HC-EBM$_2$ and R-EBM$_2$ in domain 2, with prior-learned concepts or relations in their respective domain.
\STATE {\bfseries Require:} Image $x_1$ in domain 1, containing unseen hierarchical concept $c$.
\STATE 1: $\G_1\gets \textbf{Parse}(x_1;\text{HC-EBM}_1,\text{R-EBM}_1)$ \ \ \ // \textit{Alg. \ref{alg:parsing}}
\STATE 2: $\G_2\xleftarrow{\text{channel}}\G_1$  \ \ \ \ //\ \textit{send $\G_1$ to domain 2}
\STATE 3: HC-EBM$_2(c)\gets$\,\textbf{Compose}($\G_2$; HC-EBM$_2$, R-EBM$_2$) 
\STATE \ \ \ \ // \textit{Using Hierarchical Composition Rule (Def. \ref{def:composition_rule})}
\end{algorithmic}
\end{algorithm}

We first parse the image by decomposing it $x_1$ into a graph $\G_1$ of concepts and relations previously learned by HC-EBM$_1$ and R-EBM$_1$. Next, we copy $\G_1$ to $\G_2$ in domain 2. Finally, we compose a new HC-EBM$_2(c)$ using $\G_2$, HC-EBM$_2$ and R-EBM$_2$ (Hierarchical Composition Rule, Def.~\ref{def:composition_rule}). The most complex step of this algorithm is parsing, the first step, which is described in detail in the paper (see Alg.~\ref{alg:parsing}).

\xhdr{Note on time complexity} Note that in Alg.~\ref{alg:parsing}, although the relation EBM needs to perform classification for each pair of discovered objects, which scales as $n^2$ where $n$ is the number of objects, the process is actually very fast, since we can batch all the pairs into a single minibatch, and can get the classification result with a single SGLD run, which has the same runtime as doing inference with a relation EBM on a single image.

\begin{figure}[t]
\centering
    \vspace{-2mm}
    \includegraphics[scale=0.5]{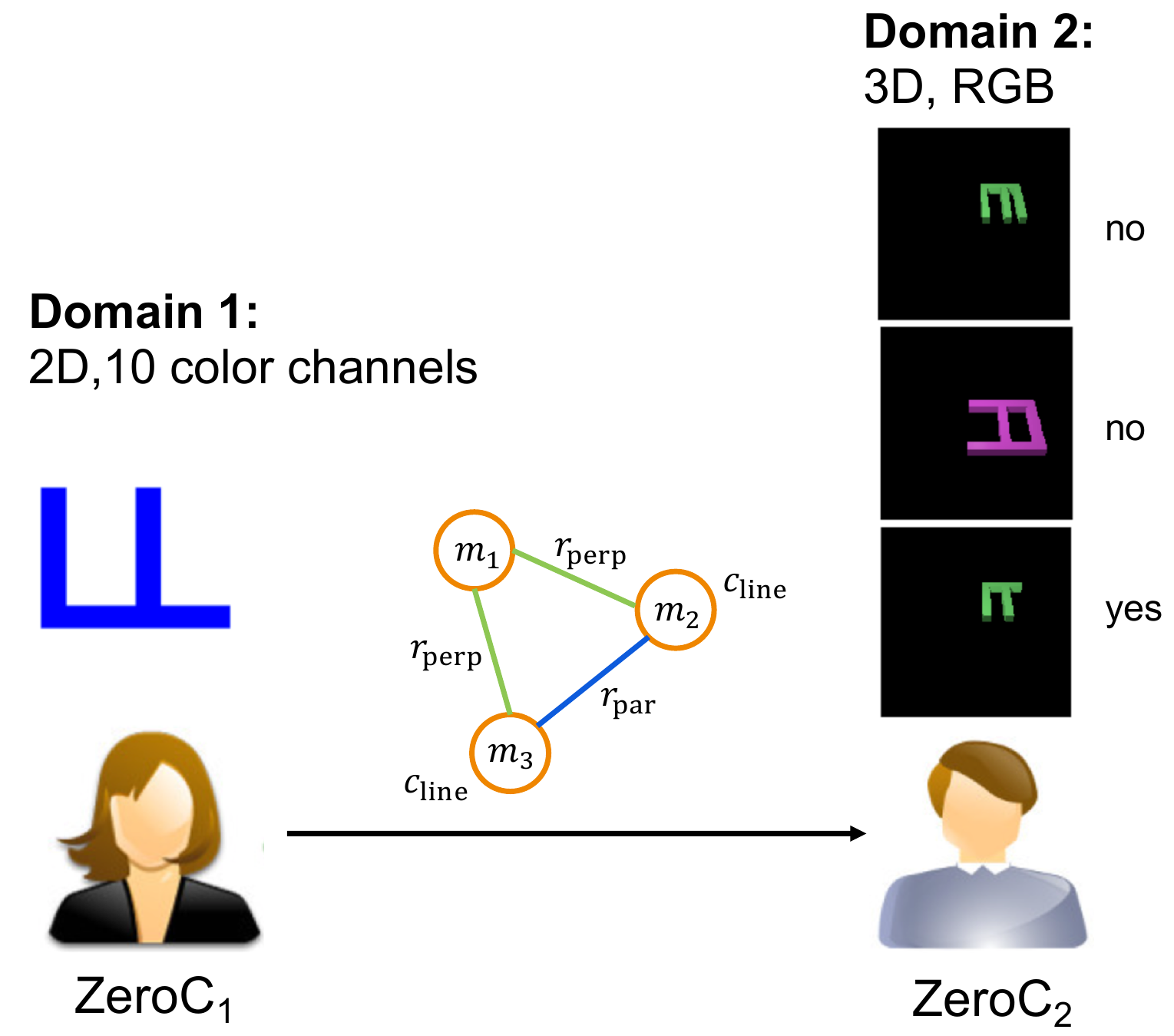}
    \caption{{\bf Acquiring Hierarchical concepts between models and domains at inference time}. ZeroC$_1$ sees the image in domain 1. It first parses it into a structure graph using Alg. \ref{alg:parsing}, then sends the graph via a communication channel to an independently trained ZeroC$_2$ in domain 2. ZeroC$_2$ can then directly classify the images in its domain.}
    \label{fig:transfer}
    % \vspace{-8mm}
\end{figure}

\subsection{Implementation Details for CADA-VAE}
\label{app:cada_vae}

CADA-VAE \cite{schonfeld2019generalized} learns a common space latent space for image features and class embeddings, by aligning modality specific variational autoencoders. Alignment is encouraged by adding two regularization terms to the standard VAE loss. This enables discriminative latent features to be sampled for unseen classes and a softmax classifier to be trained on top of such features. 

To train VAE's in encoding and decoding features in the image modality, we require a pretrained backbone specific to our domain. We obtain this pretrained model by training train a network for predicting object masks and either the concept or relation labels in a self-supervised manner \cite{rajasegaran_self-supervised_2020}. Similar to the original work, we use a ResNet-12 as the backbone, which consists of 3 residual blocks of 64, 160, 320 filters, each with $3 \times 3$ convolutions. A $2 \times 2$ max pooling operation is applied after each of the first 3 blocks. Following the blocks, we have two mask prediction heads that have the same architecture. The architecture is symmetric to the ResNet-12 backbone, with 3 blocks of 320, 160, 64 filters, each with $3 \times 3$ convolutions and an upsampling layer. For classification, a global average pooling is applied after the last block. Additionally, a 4 neuron fully-connected layer is added after the final classification layer. After training the network end-to-end, we use the ResNet backbone as a pretrained feature extractor. 

Class embeddings for the Hierarchical Concept corpus consist of slots for each atomic concept and relation. The number of slots per concept / relation is equal to the maximum number of times it can appear in a hierarchical concept. A single slot assignment (setting the value of a slot from 0 to 1) corresponds to an instance of the slot's concept / relation. Multiple slots are assigned if more than one instance of the matching concept / relation is found in a class. During training, where the ground truth label exists for only one concept / relation instance, we randomly sample at each minibatch to determine which of the ground truth slot to assign. In this way, the class embedding VAE should learn encodings that are invariant to permutations of assigned and non-assigned slots. 

\subsection{Investigation of CADA-VAE performance on HD-Letter dataset}
\label{app:cada_vae_explanation}

In Table \ref{tab:classification-detection}, we see that the classification accuracy of CADA-VAE is 18.0\%, even lower than the ``Statistics'' method of 46.5\%. Here we investigate the reason. 

Firstly, we want to see if the image encoder in CADA-VAE has enough power to differentiate the different novel concept classes (``Eshape'', ``Fshape'' or ``Ashape'') during inference. We perform t-SNE on the embeddings of test images, with each color denoting a different class (the labels are unseen to the algorithm). Fig. \ref{fig:cada_investigation} shows the visualization. We see that the encoder is able to roughly separate the unseen images into different clusters that roughly correspond to different unseen concept classes. This shows that the image encoder of CADA-VAE has enough power and is not the reason for its low performance in HD-Letter. The reason then lies in the  distribution shift of the class embedding, which is a multi-hot vector indicating the presence of individual concepts (``line'') or relations (``parallel'', ``perp-mid'', ``perp-edge''). During training, the class encoder has only seen the class embeddings where only one concept or relation is present. However, during inference, a hierarchical concept (\textit{e.g.} ``Eshape'') may have up to 4 concept instances (\textit{e.g.} 4 lines) and 6 relations, thus the class embedding will have many hots activated. This constitutes a large distribution shift for the class embedding, so that CADA-VAE does not know how to interpret it. Note that this cannot be easily fixed with alternative class embedding schemes. No matter how we specify features for the classes in training and inference, the nature of our challenging HD-Letter dataset will result in a large distribution shift for the class embedding (up to 4 concept instance + 6 relation in inference vs. 1 concept/ 1 relation in training).

This shows the limitation of CADA-VAE and similar zero-shot learning methods, where they are not equipped to handle zero-shot learning of hierarchical concepts that are more complex due to the composition of learned concepts and relations. In contrast, our \proj naturally supports such composition into concepts with more complex structure, which enables zero-shot recognition of hierarchical concepts at inference time.

\begin{figure}[t]
\centering
    \vspace{-2mm}
    \includegraphics[scale=0.45]{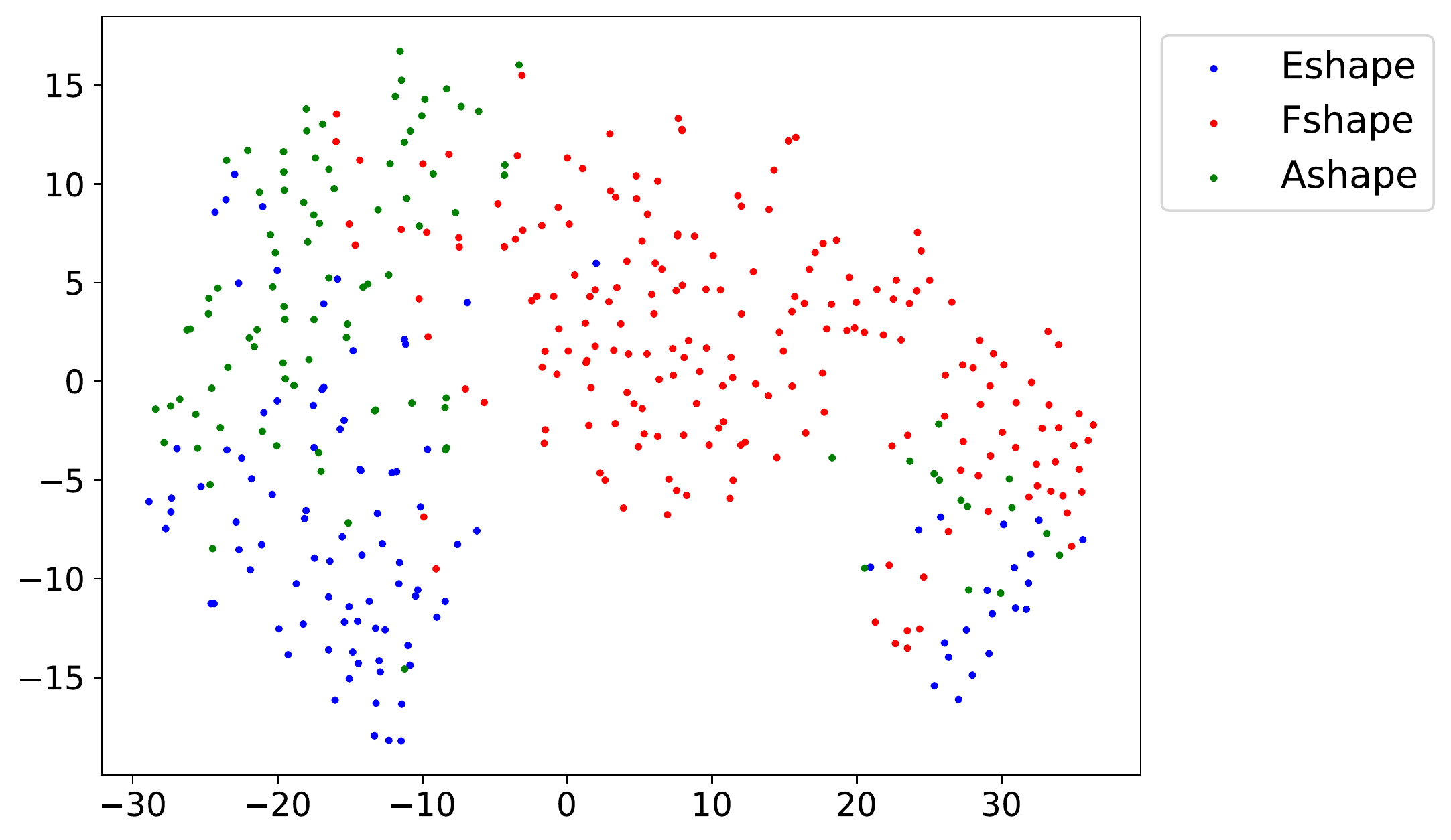}
    \caption{t-SNE visualization for the embedding of test images consisting of hierarchical concepts of ``Eshape'', ``Fshape'' or ``Ashape''. We see that the image encoder is able to roughly cluster the unseen concepts into different clusters.}
    \label{fig:cada_investigation}
    % \vspace{-8mm}
\end{figure}

\subsection{Implementation Details for Mask R-CNN + Relation Classification}
\label{app:mask_rcnn}

Firstly, one may note that the performance of Mask R-CNN's graph parsing accuracy (35.5\%) is seemingly low. In fact, this performance is very good. Note that we use the stringent metric of isomorphism accuracy: the acc is 1 \emph{only} if the inferred graph is isomorphic to ground-truth, and 0 otherwise. This metric presents a major challenge even for ``simple'' shapes. In the example of parsing the graph for ``E'' shape with 4 nodes (lines), 6 edges (their relations), even if each node and edge classification acc is 0.9, the isomorphism acc is $0.9^{10}=0.349$, and 0.8 individual acc would result in  $0.8^{10}=0.107$. Thus to reach 35.5\% isomorphism accuracy, a model has to have very high accuracy for classifying individual concepts and relations.

We perform intensive hyperparameter tuning for the baseline Mask R-CNN + relation classification. We first pretrain a Mask R-CNN \cite{he2017mask} on the Hierarchical Concept corpus to minimize object segmentation and object classification losses. We then fix the Mask R-CNN weights and use the pretrained networks to output all objects in an image. Given the Mask R-CNN output for an image, the relation head is trained to minimize the relation classification loss with respect to a ground truth pair of objects and their corresponding relation. 

\textbf{Mask R-CNN}: Instance segmentation is the task of precisely detecting objects through bounding-box localization and precisely segmenting each object instance while correctly predicting its corresponding class. We make use of a popular choice for instance segmentation, Mask R-CNN, which consists of a branch for predicting segmentation masks within Region of Interests (RoI) in parallel with a branch for classification and bounding box regression. The architecture for the segmentation branch consists of four convolutional upsampling blocks, each with 256 filters. 

A vanilla Mask R-CNN architecture with a ResNet50 and Feature Proposal Network (FPN) \cite{lin_feature_2017} backbone achieves strong object localization and segmentation results on the COCO dataset \cite{lin_microsoft_2014} when trained from scratch. As COCO is much larger in scale, in terms of image resolution, number of object instances, and number of object classes, we modify the vanilla architecture to suit the Hierarchical Concept corpus's smaller image size. We use a ResNet18-FPN backbone for feature extraction, with the FPN having anchor boxes with side lengths of 4, 8, 16, 32. During training, we consider the top 2000 object proposals and reduce these to the top 1000 through NMS with a IoU threshold of 0.7. During inference, the top 20 object proposals are kept before and after NMS. We train on one GPU with a batch size of 2 for 176k iterations (equivalent to 8 epochs), with a learning rate of 0.0025. 

To tune the standard Mask R-CNN architecture that works on much larger image resolutions, we decreased the size of the backbone ResNet, decreased the size of anchor boxes, and increased the scales for the RoIPooler relative to the input image. We found that decreasing the number of proposals before NMS did not improve performance, nor did decreasing the size of the mask upsampling head. We also tuned the learning rate across values in the set \{0.005, 0.001, 0.0025, 0.05\}, with 0.0025 yielding the best performance. Our Mask R-CNN performance for detecting lines on a subset of 1000 training images is 97.4 mIoU. We found that this performance is not significantly affected  by increases in learning rate warmup length or by increases in the number of iterations before learning rate decay.

\textbf{Relation Head:} The architecture of the relation classifier consists of three residual blocks, followed by a fully-connected network with two hidden layer. Each residual block consists of 3x3 convolutions with spectral normalization, followed by downsampling. The relation head predicts the relation between all pairs of object masks outputted by the Mask R-CNN. To train the relation head, we obtain the predicted object masks that are closest (in terms of IoU) to the pair of ground truth masks and compute the loss on the predicted relation for the corresponding Mask R-CNN object masks. We tuned the learning rate of the relation head across values in the set \{5e-5, 2.5e-5, 1e-5\} and found that 2.5e-5 gave the best performance. We found that fixing the Mask R-CNN weights was essential to stable training of the relation head. The accuracy of our relation classifier is 94.5\% on the training set. 

\subsection{The Hierarchical Concept Corpus}
\label{app:dataset}

\begin{table}[t]
\centering 
\resizebox{0.65\linewidth}{!}{%
\begin{tabular}{rl}\hline
  \textbf{Concept} & \textbf{Description} \\ \hline
  line  & an object represents a solid line \\
  rectangle  & an object represents a hollow rectangle \\
  rectangleSolid & an object represents a solid rectangle \\
  L-shape & an object represents a shape that is ``L'' like \\
  C-shape & an object represents a shape that is ``C'' like \\
  A-shape & an object represents a shape that is ``A'' like \\
  E-shape & an object represents a shape that is ``E'' like \\
  T-shape & an object represents a shape that is ``T'' like \\
  rand-Shape & an object that is randomly constructed \\
  \hline
  \end{tabular}
}
\\
\resizebox{0.65\linewidth}{!}{%
\begin{tabular}{rl}\hline
  \textbf{Relation} & \textbf{Description} \\ \hline
  inside  & object x is inside of object y \\
  enclose & objects x is enclosed by object y \\
  parallel  & objects x is parallel with object y \\
  perp-mid  & object x is perpendicular with object y, \\
   & and touch object y in its middle. \\
  perp-edge & object x is perpendicular with object y, \\
   & and touch object y in its edge. \\
  non-overlap  & objects x and y are not overlapped in \\
  & both x and y axes. \\
  \hline
  \end{tabular}
}
\caption{Supported primitive concepts and relations in our Hierarchical-Concept corpus.}\label{tab:concept-relation}
\end{table}

In this section, we describe our Hierarchical-Concept corpus for our experiments and how it is generated. Our data generation framework is designed for generating large-scale datasets for concept and relation learning in a grid-world setting. Specifically, it samples pixel-level arbitrary objects, and places on to grid-world with predefined relations between objects. The task format is inspired by the Abstract Reasoning Corpus (ARC) proposed by \cite{chollet2019measure}. 

\Figref{fig:hc-letter-examples} to \Figref{fig:hd-examples} shows  examples for concepts and relations used in our experiments, for the HD-Letter and HD-Concept datasets, and a dataset of 2D to 3D concept transfer that tests acquiring hierarchical concepts between domains at inference time. For both HD-Letter and HD-Concept, there are 44000 examples of concepts, split 10:1 for train and validation, and 44000 examples for relations, split 10:1 for train and validation. At inference time for hierarchical concepts, the classification task has 200 examples, and detection task has 600 examples, 200 for each of the hierarchical concepts. The 2D to 3D transfer dataset has  200 tasks. We define our primitive objects and relations in \Tabref{tab:concept-relation}. 

To generate the datasets\footnote{
The code for generating the dataset can be found at project website  \url{http://snap.stanford.edu/zeroc/}.}, our engine consists of the following components:

\paragraph{Concept (Object)} As shown in~\Tabref{tab:concept-relation}, we define several shape primitives such as line, reactangle, L-shape, and random shape. Our data generation framework allows configurations such as color, width, height and orientation. We allow nine different colors, and four different orientations at maximum. In addition, we have composite objects as well. For instance, the ``Lshape'' consists of two lines with a fixed relation between them. We evaluate our models with these composite objects during inference time to evaluate their performance on recognizing composite objects based on primitive concepts. 

\paragraph{Relation}  As shown in~\Tabref{tab:concept-relation}, we define several relation primitives such as ``inside'', ``enclose'', and ``parallel''. We use these relations to define the spatial relations between objects. For example, ``inside'' means one object is inside of the other object. A pair of objects may formulate multiple relations between them. Likewise, a pair of objects may be \emph{unrelated} as well giving there is no primitive relation between them. 

\paragraph{Canvas} Each of our examples in the corpus consists a canvas, where objects are positioned with relations between them. Our canvas is a $n \times n$ grid world, where each pixel in the grid world is a colored pixel. Our data generation framework can place objects onto the canvas with desired relations. For example, our framework can generate two objects where one of them is inside of the other. In the process, our framework samples a rectangle, and sample another object to be placed inside the rectangle. In the meantime, the framework allows multiple object pairs to be defined when placing objects. In addition, the framework allows objects to be specified with pre-defined attributes including shape, color, width, height and orientation. We also allow other configurations such as whether we allow objects to touch each other, and whether objects have unified color. 

\paragraph{Generation Process and Artifacts} To generate training and evaluation sets for our experiments, we specify configurations of the canvas and repeatedly randomly generate canvas till we have enough examples. In addition, we allow distractor sampling, where we  specify the number and the shape of distractors. Then, our data generation framework places distractors at random on the canvas. Our framework parses the canvas and the relations between objects after all objects are placed onto the canvas. Notice that the set of the relations between objects after the placement is a superset of the pre-defined configuration when creating the canvas. For example, if we sample two objects sharing the same color along with another random distractor. The relation between the distractor and one of the object can be free-formed. We disallow non-flatten configuration when creating canvas. For example, if there are three objects to be placed, we disallow circular relations specified between them (\textit{i.e.}, there exist a relation between each two objects). We disallow two objects touching each other share the same color. To generate 3D images, we first generate the 2D images, and then use a standard tool of  povray\footnote{\url{https://github.com/POV-Ray/povray/tree/latest-stable}} to build 3D scenes based on the contour of the 2D image. We make sure that in the 2D to 3D transfer dataset, there is no overlap between the generated 2D images and the images used to generate 3D images by using a distinct seed.

\paragraph{Design of dataset HD-Concept} The main goal of this dataset is to test whether our method can detect different relational graph structure given the same number of concept nodes. Concretely, given 2 ``rectangle'' concept instances and one ``E-shape'' instance, there are limited ways to form a compositional concept with different relational graph structures: Concept1 is where ``E-shape'' is not inside any of the ``rectangle''. Concept2 is ``E-shape'' and ``rectangle''$_1$ are both inside the other ``rectangle''$_2$, and ``E-shape'' not inside ``rectangle''$_1$. Concept3 is that ``Eshape'' is enclosed by one ``rectangle'' which is also enclosed by another ``rectangle''. All three compositional concepts have the same number of constituent concepts but different relation graphs. 

Example images datasets are shown in Fig. \ref{fig:hc-letter-examples}, Fig. \ref{fig:hc-concept-examples} and Fig. \ref{fig:hd-examples}.

\label{app:baselines}

\begin{figure}[t!]
\centering
     \includegraphics[width=0.95\textwidth]{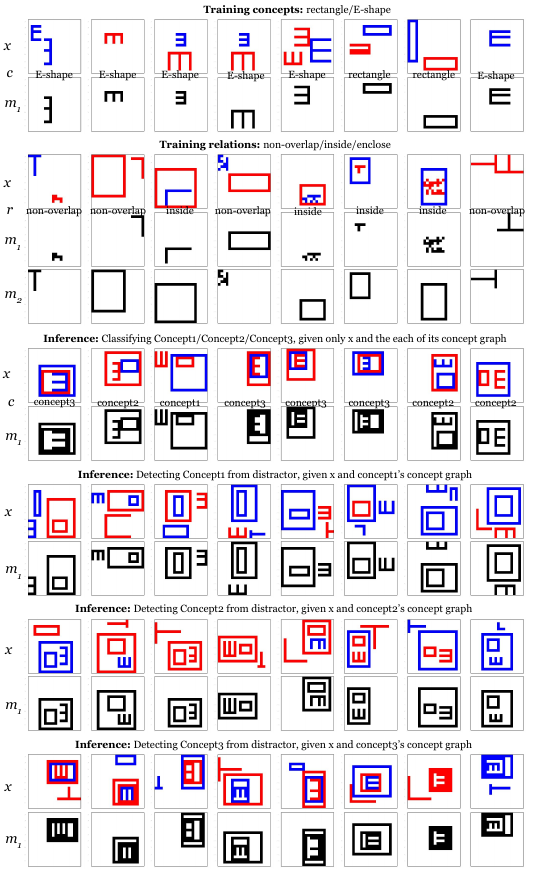}
      \caption{Samples examples from the HD-Concept dataset for training and inference. The models are trained with concept of ``E-shape'' and ``rectangle'', and relations of ``non-overlap'', ``inside'', ``enclose'' (inverse in ``inside'', if exchanging two masks). At inference, the models are tasked to perform classification and pixel-wise detection, on hierarchical concepts (w.r.t. training) of ``Concept1'', ``Concept2'' and ``Concept3'' (see the bottom 3 panel). In this dataset, during training relation, the pair of objects does not appear in inference, testing if the relation can generalize to new objects.}
       \label{fig:hc-concept-examples}
\end{figure}

\begin{figure}[t!]
\centering
     \includegraphics[width=1\textwidth]{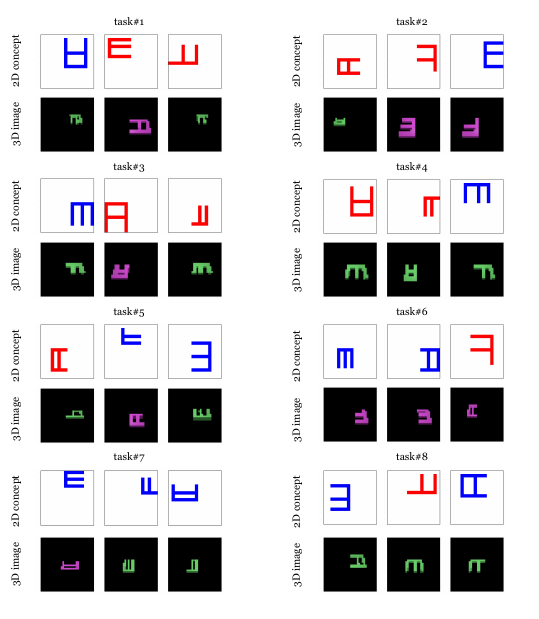}
      \caption{ Examples from our Hierarchical-Concept corpus for experiments of acquiring hierarchical concepts between domains. In each panel, the upper three images demonstrate the concepts in domain 1, and after communicating high-level knowledge, independently trained models in domain 2 need to perform classification and detection on its own domain.}
       \label{fig:hd-examples}
\end{figure}

\subsection{Limitations of current work}
\label{app:limitations}

In this work, we have demonstrated the zero-shot concept recognition and acquisition capability of \proj, with experiments in a grid-world domain. We focus on grid-world since it provides a systematic and challenging testbed to evaluate the above two capabilities of models, similar to many other pioneering works that evaluate their models in grid-world that captures the essence of the problem (\textit{e.g.} RL in early days \cite{sutton1998sutton}, PrediNet \cite{shanahan2020explicitly}, program synthesis evaluated with Karel dataset \cite{shin2018improving}, BabyAI \cite{chevalier2018babyai}, Machine Learning Theory of Mind \cite{rabinowitz2018machine}, etc.) Nevertheless, the fact that our main experiment is in grid-world is a limitation of current work. In the main Sec \ref{sec:acquiring_exp}, we train ZeroC$_2$ with 3D images, and the fact that ZeroC$_2$ is able to zero-shot classify and detect in its domain demonstrate that \proj is able to handle more complex 3D images. Moreover, in Appendix \ref{app:additional_exp}, we perform additional zero-shot classification with a variant of CLEVR, which shows that our \proj is able to handle more realistic 3D images, and out-perform strong baseline of CADA-VAE.

Another limitation of current work is that we have only considered one hierarchy of composition, without considering more hierarchies. Although more hierarchies is in principle possible with the \proj method, it is out of scope of this work, since in this work we focus on demonstrating \emph{whether} zero-shot recognition and acquisition of concepts are possible with our model. It is an exciting future work to build on current work, to explore zero-shot recognition with multiple levels of hierarchies. 

Our \proj architecture naturally supports continued expansion or compression of the EBM pool, as newly learned compositional concept EBMs can be dynamically added to the pool. Independently trained EBMs on new concepts/relations can also be added to the pool and composed together with existing EBMs. This is an exciting future direction, but is out-of-scope of the present paper, since here we focus on proposing the framework and demonstrate the zero-shot recognition and acquisition capability of \proj.

\subsection{Broader social impact}
\label{app:social_impact}

Here we discuss the broader social impact of our work, including its potential positive and negative aspects. On the positive side, the capability of \proj enables will improve the generalization capability of deep learning models, allowing them to acquire concepts and address more diverse tasks at inference time. This provides a possible method to address the long-standing problem that deep learning models has limited generalization capability and mainly learn via examples. Our \proj also improves interpretability of models, since we can know exactly the structure of the concepts the model learns, allowing us to know \emph{how} the model makes such decision.

We see no obvious negative social impact of this work. In its current state, \proj's capability is still very limited, not nearly addressing any of the tasks as good as human level. It will only improve the interpretability and versatility of the models, which can be used to better address challenging tasks in society.

\subsection{Additional experiments}
\label{app:additional_exp}

\begin{figure}[h!]
\centering
    \includegraphics[width=0.8\textwidth]{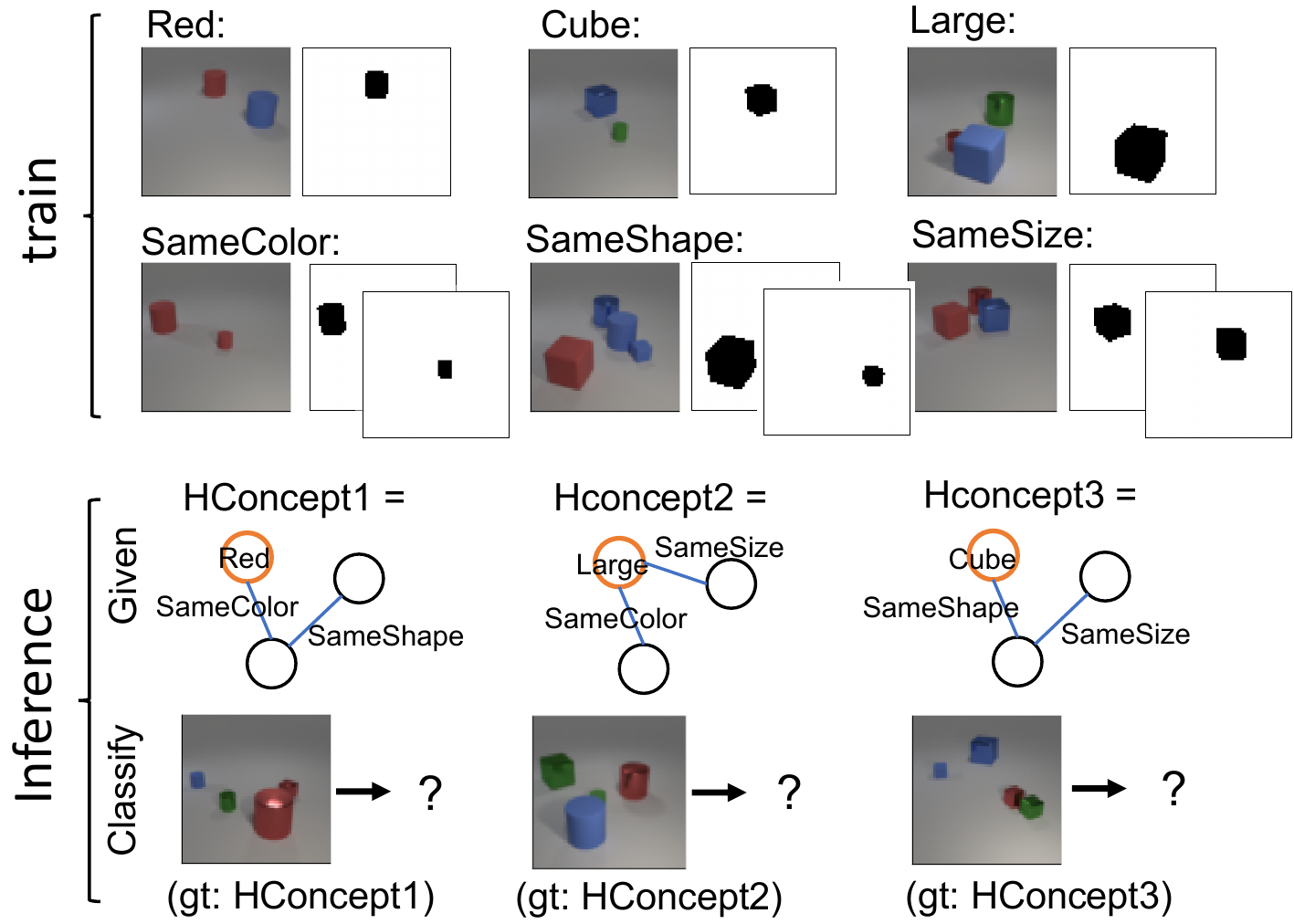}
\caption{Zero-shot classification task based on CLEVR. Here during training, the models are given only tuples of ($image$, $mask(s)$, $label$) for elementary concepts (Red, Cube, Large) or elementary relations (SameColor, SameShape, SameSize). In inference, the models are asked to perform zero-shot classification of a hierarchical concept. For example, in lower left figure, the models are asked to classify whether an image contains HConcept1 (True if the image contains three objects where one object is red, the second object has the same color as the first object and the third object has the same shape as the second object, and False otherwise.) HConcept2 and HConcept3 have similar form of interpretation. This is a challenging task since in inference, the concepts to classify are more complicated than that in training. }
\label{fig:clevr}
\end{figure}

Here we perform additional experiments with a CLEVR dataset to explore whether \proj can generalize to more realistic images. The models are trained with basic concepts of ``Red'', ``Cube'', ``Large'' and relations of ``SameColor'', ``SameShape'', ``SameSize'' (See Fig. \ref{fig:clevr}), where the models are provided with tuples of ($image$, $mask(s)$, $label$) for concept or relation. During inference, given only symbolic specification of three hierarchical concepts, the model needs to zero-shot classify whether an image contains an instance of such concept. For example, HConcept1 is defined as three objects where one object is red, the second object has the same color as the first object and the third object has the same shape as the second object. The other two hierarchical concepts have a similar structure of specification. We used 100K images for training and 200 for inference (evaluation). The table  below shows the results of our model and comparison with CADA-VAE and the statistics baseline. We see that \proj achieves significantly higher performance on the task than the baselines, and able to zero-shot classify more realistic images.

\begin{table}[h!]
\centering
\small
\resizebox{0.5\textwidth}{!}{
\begin{tabular}{c|c}
\hline
Model    & Classification acc (\%)   \\ \hline
Statistics & 33.4  \\ 
CADA-VAE & 45.3 \\ \hline
\proj (ours) & \textbf{56.0} \\\hline
\end{tabular}
}
\label{tab:clevr}
\end{table}

\subsection{Generality of ZeroC approach}
\label{app:generality}
The ZeroC framework is quite general. The generality lies in the following two folds:

\paragraph{Generality of learning the elementary concepts and relations} The EBMs in the ZeroC framework can learn general primitives as long as labeled data is provided to demonstrate the concept or the relation, even if that concept or relation is a range that contains some intrinsic variation. For example, to learn the ``acute angle'' relation between two lines (with varying angles), ZeroC only needs a dataset that contains many (image, mask1, mask2, ``accute-angle'') tuples where the mask1 and mask2 identify the two lines in the image that form an acute angle, with different examples containing different angles. In other words, as long as the dataset contains enough data that identifies a concept/relation in a certain range, the ZeroC can learn such primitives. This is also shown in the HD-Concept dataset in Section 3.1, where ``inside'', ``outside'', ``non-overlap'' relation primitives are learned, and each relation has intrinsic variation. For example, the two masks for ``inside'' relation can have different positions, relative positions, and sizes.

\paragraph{Generality to different datasets and scenarios}
Our architecture is general to learn diverse concepts and relations. For all our experiments, we use the \emph{same} network architecture (Appendix \ref{app:hicone_architecture}), for the (1) dataset HD-Letter, (2) dataset HD-Concept (that contains more complex concepts and relations e.g. ``rectangles'', ``Eshape'', ``inside'', ``enclose''), (3) Section \ref{sec:acquiring_exp} ``Acquiring Novel Hierarchical Concepts Across Domains'', and (4) CLEVR experiment in Appendix \ref{app:additional_exp}. This shows that the algorithm is very general, not tuned toward specific concepts or scenarios. The architectures only differ in the number of input channels (since the 2D images have 10 channels and 3D images have 3 RGB channels). For future work, we can also experiment with using a single \emph{model} to learn concepts or relations across datasets, which is out-of-scope of current work.

\subsection{Scalability of ZeroC}
\label{app:scalability}

\paragraph{Scalability to task complexity}
We have demonstrated that even for larger images like the 3D image (32x32x3) in Section \ref{sec:acquiring_exp} and CLEVR in Appendix \ref{app:additional_exp} (64x64x3), our approach achieves reasonable accuracy, significantly outperforming baselines. This shows the scalability of our methods to larger images and realistic use cases. In addition, downsampling the images to lower resolution can be performed to reduce the complexity of the inference and learning.

\paragraph{Scalability in terms of time complexity}

In terms of time complexity, the SGLD inference algorithm (Alg. \ref{alg:SGLD}) uses a fixed number of iteration steps $K$. We find that $K=60$ is enough for reasonable detection accuracy for larger images. For larger images, a single step of SGLD make take slightly longer to run due to the larger image size. For parsing hierarchical concept from image (Alg. \ref{alg:parsing}), the detection of all concept instances can be obtained for a single SGLD run, and for the classification of relations, we can concatenate all pairs of concept masks into a single minibatch and feed into the relation-EBM, which only requires one relation-EBM forward run, which is instant. Thus we see that the time complexity is fairly constant for increasing number of objects in the images and larger image size, measured in terms of a single forward or SGLD step.

\paragraph{Potential application to real world images} In order to scale to recognizing real-world objects, the ZeroC architecture and pipeline will be able to do that in principle. The main bottleneck is presented by labeled data, as we need to have detailed labels for many elementary concepts and relations that constitute real world objects. For example, the CUB-200-2011 dataset \cite{wah2011caltech} provides annotations for elementary concepts for birds, e.g. beak, belly, tail, etc. This dataset lacks relation annotations, so is unsuitable for our pipeline. A suitable dataset for our pipeline can be like an augmented dataset to the above CUB-200-2011 dataset that also has relation annotations like ``connect-to'', ``up'', ``down'', ``extend'', etc. With such annotations, we can learn both concept EBMs and relation EBMs and compose together to recognize compositional concepts like different species of birds.

\subsection{Computational complexity of ZeroC}
\label{app:complexity}

The computational complexity of the inference is detailed in Appendix \ref{app:scalability}. In summary, the run time remains fairly constant increasing number of objects in the images and larger image size, measured in terms of a single forward or SGLD step. Since learning take the inference as an inner loop, it also remains fairly constant for the same model structure. For increasing EBM model size, the increased number of parameters will definitely require more time to train, as is typical for deep learning models. Empirically, we observe that for the same model architecture, it takes similar number of epochs to learn reasonably for different datasets.

\end{document}